\documentclass[11pt]{article}
\usepackage[margin=1in]{geometry}
\usepackage{parskip}

\usepackage[round]{natbib}

\usepackage{url}            %
\usepackage{booktabs}       %
\usepackage{amsfonts}       %
\usepackage{nicefrac}       %
\usepackage{microtype}      %
\usepackage{enumitem}
\setenumerate[1]{label=(\arabic*)}

\usepackage{subcaption}

\usepackage{graphicx}
\usepackage{color}
\usepackage{hyperref}%
\hypersetup{colorlinks=true, linkcolor=blue, citecolor=blue, urlcolor=blue}
\usepackage{algorithm, algorithmic}
\usepackage{amsmath,amsthm,amssymb}
\usepackage{bm}
\newcommand{\RR}{\mathbb{R}}
\newcommand{\CC}{\mathbb{C}}
\newcommand{\BB}{\mathbb{B}}

\renewcommand{\SS}{\mathbb{S}}
\newcommand{\HH}{\mathbb{H}}
\newcommand{\NN}{\mathbb{N}}
\newcommand{\ZZ}{\mathbb{Z}}

\newcommand{\PP}{\mathbb{P}}
\newcommand{\FF}{\mathbb{F}}
\newcommand{\dd}{\mathrm{d}}

\newcommand{\calS}{\mathcal{S}}
\newcommand{\calF}{\mathcal{F}}
\newcommand{\calG}{\mathcal{G}}
\newcommand{\calH}{\mathcal{H}}

\renewcommand{\aa}{{\bm{a}}}
\newcommand{\bb}{{\bm{b}}}
\newcommand{\cc}{{\bm{c}}}
\newcommand{\uu}{{\bm{u}}}
\newcommand{\vv}{{\bm{v}}}
\newcommand{\xx}{{\bm{x}}}
\newcommand{\yy}{{\bm{y}}}
\newcommand{\zz}{{\bm{z}}}
\newcommand{\oo}{{\bm{o}}}
\newcommand{\ee}{\bm{e}}
\newcommand{\xxi}{{\bm{\xi}}}
\newcommand{\oomega}{{\bm{\omega}}}
\newcommand{\llambda}{{\bm{\lambda}}}
\renewcommand{\ss}{{\bm{s}}} %

\newcommand{\rrho}{{\bm{\rho}}}

\newcommand{\iprod}[1]{\langle#1\rangle}
\newcommand{\iiprod}[1]{(\!(#1)\!)}

\DeclareMathOperator{\hull}{span}
\DeclareMathOperator{\tr}{tr}
\DeclareMathOperator{\rank}{rank}
\DeclareMathOperator{\diag}{diag}

\newcommand{\subH}{\calH_m}

\newcommand{\lieG}{\mathfrak{g}}
\newcommand{\lieK}{\mathfrak{k}}
\newcommand{\lieA}{\mathfrak{a}}
\newcommand{\lieN}{\mathfrak{n}}
\newcommand{\bdX}{\partial X}

\newcommand{\eprod}[1]{(#1)}
\newcommand{\bdD}{\partial D}

\newcommand{\bdB}{\partial\BB}
\newcommand{\bdP}{\partial\PP}

\newcommand{\sym}{\mathfrak{S}}
\newcommand{\pdm}{\mathfrak{P}}

\newcommand{\refeq}[1]{\eqref{eq:#1}}
\newcommand{\reffig}[1]{Figure~\ref{fig:#1}}
\newcommand{\reftab}[1]{Table~\ref{tab:#1}}
\newcommand{\refapp}[1]{Appendix~\ref{sec:#1}}
\newcommand{\refsec}[1]{\S~\ref{sec:#1}}
\newcommand{\refthm}[1]{Theorem~\ref{thm:#1}}
\newcommand{\reflem}[1]{Lemma~\ref{lem:#1}}

\theoremstyle{plain}
\newtheorem{theorem}{Theorem}[section]
\newtheorem{lemma}{Lemma}[section]
\newtheorem{example}{Example}[section]

\theoremstyle{definition}
\newtheorem{definition}{Definition}[section]

\theoremstyle{remark}
\newtheorem{remark}{Remark}

\newcommand{\step}[1]{\textbf{Step~#1}}

\newcommand{\xlabel}[1]{\label{#1}}

\newcommand{\nxrefeq}[1]{\refeq{#1}}
\newcommand{\nxreftab}[1]{\reftab{#1}}
\newcommand{\nxrefapp}[1]{\refapp{#1}}
\newcommand{\nxrefsec}[1]{\refsec{#1}}
\newcommand{\nxrefthm}[1]{\refthm{#1}}
\newcommand{\nxreflem}[1]{\reflem{#1}}

\newcommand{\coloneq}{:=}
\newcommand{\eqcolon}{=:}
\newcommand{\svert}{|}

\newcommand{\mBiggvert}{{\ \Bigg| \ }}

\title{\textbf{A unified Fourier slice method to derive ridgelet transform for a variety of depth-2 neural networks}}

\author{%
\textbf{Sho Sonoda}${}^1$ \hfill\small{\url{sho.sonoda@riken.jp}}\\
\textbf{Isao Ishikawa}${}^{2,1}$ \hfill\small{\url{ishikawa.isao.zx@ehime-u.ac.jp}}\\
\textbf{Masahiro Ikeda}${}^1$ \hfill\small{\url{masahiro.ikeda@riken.jp}}\\
\small{${}^1$\textit{RIKEN Center for Advanced Intelligence Project (AIP)}} \hfill\phantom{a}\\
\small{${}^2$\textit{Center for Data Science, Ehime University (CDSE)}}\hfill\phantom{a}\\
\phantom{\large{aaaaaaaaaaaaaaaaaaaaaaaaaaaaaaaaaaaaaaaaaaaaaaaaaaaaaaaaaaaaaaaaaaaaaaaa}}%
}

\date{April 18, 2024}%

\begin{document}

\maketitle

\begin{abstract}
To investigate neural network parameters, it is easier to study the distribution of parameters than to study the parameters in each neuron. The ridgelet transform is a pseudo-inverse operator that maps a given function $f$ to the parameter distribution $\gamma$ so that a network $\mathtt{NN}[\gamma]$ reproduces $f$, i.e. $\mathtt{NN}[\gamma]=f$. For depth-2 fully-connected networks on a Euclidean space, the ridgelet transform has been discovered up to the closed-form expression, thus we could describe how the parameters are distributed. However, for a variety of modern neural network architectures, the closed-form expression has not been known. In this paper, we explain a systematic method using Fourier expressions to derive ridgelet transforms for a variety of modern networks such as networks on finite fields $\mathbb{F}_p$, group convolutional networks on abstract Hilbert space $\mathcal{H}$, fully-connected networks on noncompact symmetric spaces $G/K$, and pooling layers, or the $d$-plane ridgelet transform.
\end{abstract}

 \section{Introduction}
 \xlabel{sec:intro}
Neural networks are learning machines that support today's AI technology. Mathematically, they are nonlinear functions determined by a network of functions with learnable parameters (called \emph{neurons}) connecting in parallel and series. Since the learning process is automated, we do not fully understand the parameters 
obtained through learning. 
An integral representation is a powerful tool for mathematical analysis of these parameters.
One of the technical difficulties in analyzing the behavior of neural networks is that their parameters are extremely nonlinear. An integral representation is a method of indirectly analyzing the parameters through their distribution, rather than directly analyzing the parameters of each neuron. The set of all the signed (or probability) parameter distributions forms a linear (or convex) space, making it possible to perform far more insightful analysis than directly analyzing individual parameters.

For instances, 
characterization of neural network parameters such as
the \emph{ridgelet transform}~\citep{Murata1996,Candes.PhD,Sonoda2015acha}
and the \emph{representer theorems} for ReLU
networks~\citep{Savarese2019,Ongie2020,Parhi2021,Unser2019},
and 
convergence analysis of stochastic gradient descent (SGD) for deep learning
such as the \emph{mean field
theory}~\citep{Nitanda2017,Chizat2018,Mei2018,Rotskoff2018,Sirignano2020lln}
and the  \emph{infinite-dimensional Langevin
dynamics}~\citep{Suzuki2020langevin,Nitanda2022langevin},
have been developed using integral representations.

\subsection{Integral Representation}

The integral representation of 
a \emph{depth-2 fully-connected neural network} 
is defined as below.
\begin{definition}
Let $\sigma:\RR\to\CC$ be a measurable function, called \emph{activation function}, and fix it. For any signed measure $\gamma$ on $\RR^m\times\RR$, called a \emph{parameter distribution}, we define the \emph{integral representation of depth-2 fully-connected neural network} as
\begin{equation}
S[\gamma](\xx) = \int_{\RR^m\times\RR} \gamma(\aa,b)\sigma(\aa\cdot\xx-b)\dd\aa\dd b, \quad \xx \in \RR^m. \xlabel{eq:contiSNN}
\end{equation}
\end{definition}

Here, for each hidden parameter $(\aa,b)$, 
feature map $\xx\mapsto\sigma(\aa\cdot\xx-b)$ corresponds to a single hidden neuron with activation function $\sigma$,
weight $\gamma(\aa,b)$ corresponds to an output coefficient,
and the integration implies that all the possible neurons are assigned in advance.
Since the \emph{only} free parameter is \emph{parameter distribution} $\gamma$, we can identify network $S[\gamma]$ with point $\gamma$ in a function space.

We note that this representation covers \emph{both} infinite (or continuous) and finite widths. Indeed, while the integration may be understood as an infinite width layer, the integration with a finite sum of point masses such as $\gamma_p \coloneq  \sum_{i=1}^p c_i \delta_{(\aa_i,b_i)}$ can represent a finite width layer:
\[
    S[\gamma_p](\xx) = \sum_{i=1}^p c_i \sigma( \aa_i \cdot \xx - b_i ) = C\sigma(A\xx-\boldsymbol{b}), \quad \xx \in \RR^m\]
where the third term is the so-called ``matrix'' representation with matrices $A \in \RR^{p\times m}, C \in \RR^{1\times p}$ and vector $\boldsymbol{b} \in \RR^p$ followed by component-wise activation $\sigma$.
Singular measures as above can be mathematically justified without any inconsistency if the class of the parameter distributions are set to a class of Borel measures or Schwartz distributions. 

There are at least four advantages to introducing integral representations:
\begin{enumerate}
    \item Aggregation of parameters, say $\{(\aa_i,b_i,c_i)\}_{i = 1}^p$, into a single function (parameter distribution) $\gamma(\aa,b)$,
    \item Ability to represent finite models and continuous models in the same form,
    \item Linearization of networks and convexification of learning problems, and
    \item Presence of the ridgelet transform.
\end{enumerate}
Advantages 1 and 2 have already been explained. Advantage 3 is described in the next subsection, and Advantage 4 is emphasized throughout the paper.
On the other hand, there are two disadvantages:
\begin{enumerate}
    \item Extensions to deep networks are hard\footnote{Good News: After the
     initial submission of this manuscript, the authors have
  successfully developed the ridgelet transform for \emph{deep}
  networks in~\citet{sonoda2023deepridge,sonoda2023joint}.}, and
    \item Advanced knowledge on functional analysis are required.
\end{enumerate}

\subsection{Linearization and Convexification Effect}

The third advantage of the integral representation is 
the so-called \emph{linearization} (and \emph{convexification}) tricks. That is, while the network is \emph{nonlinear} with respect to the raw parameters $\aa$ and $b$, namely,
\[
    S[\delta_{(\alpha_1 \aa_1 + \alpha_2 \aa_2,b)}] \neq \alpha_1 S[\delta_{(\aa_1,b)}] + \alpha_2  S[\delta_{(\aa_2,b)}], \quad \alpha_1,\alpha_2 \in \CC
\]
(and similarly for $b$),
it is \emph{linear} with respect to the parameter distribution $\gamma$, namely,
\[
    S[\alpha_1 \gamma_1 + \alpha_2 \gamma_2] = \alpha_1 S[\gamma_1] + \alpha_2 S[\gamma_2], \quad \alpha_1,\alpha_2 \in \CC.
\]

Furthermore, linearizing neural networks leads to convexifying learning problems. Specifically, for a convex function $\ell : \RR \to \RR$, the loss function defined as $L[ \gamma ] \coloneq  \ell( S[\gamma] )$ satisfies the following:
\[
L[ t \gamma_1 + (1-t) \gamma_2 ] \le t L[\gamma_1] + (1-t) L[\gamma_2], \quad t \in [0,1].
\]

It may sound paradoxical that a convex loss function on a function space has local minima in raw parameters, but we can understand this through the chain rule for functional derivative: Suppose that a parameter distribution $\gamma$ is parametrized by a raw parameter, say $\theta$, then
\[
\frac{\partial L[ \gamma(\theta) ]}{\partial \theta} = \left\langle\frac{\partial \gamma(\theta)}{\partial \theta}, \frac{\partial L[\gamma]}{\partial \gamma} \right\rangle.
\]
In other words, a local minimum ($\partial_\theta L=0$) in raw parameter $\theta$ can arise not only from the global optimum ($\partial_\gamma L=0$) but also from the case when two derivatives $\partial_\theta \gamma$ and $\partial_\gamma L$ are orthogonal.

The trick of lifting nonlinear objects in a linear space has been studied since
the age of Frobenius, one of the founders of the linear representation theory of groups.
In the context of neural network study, as well as the recent studies mentioned
above, either the integral representation by~\citet{Barron1993} or the convex
neural network by~\citet{LeRoux2006.convex} are often referred.
In the context of deep learning theory, this linearization/convexification trick has been employed to show the global convergence of the SGD training of shallow ReLU networks 
\citep{Nitanda2017,Chizat2018,Mei2018,Rotskoff2018,Sirignano2020lln,Suzuki2020langevin,Nitanda2022langevin},
and to characterize parameters in ReLU
networks~\citep{Savarese2019,Ongie2020,Parhi2021,Unser2019}.

\subsection{Ridgelet Transform}
 \xlabel{sec:ridgelet.classic}
The fourth advantage of the integral representation is the so-called the \emph{ridgelet transform} $R$, or a right inverse operator of the integral representation operator $S$. For example, the ridgelet transform for depth-2 fully-connected network \nxrefeq{contiSNN} is given as below.

\begin{definition}
For any measurable functions $f : \RR^m \to \CC$ and $\rho : \RR \to \CC$,
\begin{equation}
    R[f;\rho](\aa,b) \coloneq  \int_{\RR^m} f(\xx) \overline{\rho(\aa\cdot\xx-b)}\dd\xx, \quad (\aa,b) \in \RR^m\times\RR. \xlabel{eq:ridgelet}
\end{equation}
\end{definition}

In principle, the \emph{ridgelet function} $\rho$ can be chosen independently of the activation function $\sigma$ of neural network $S$. The following theorem holds.

\begin{theorem}[Reconstruction Formula]\xlabel{thm:reconst}
Suppose $\sigma$ and $\rho$ are a tempered distribution ($\calS'$) on $\RR$ and a rapidly decreasing function ($\calS$) on $\RR$, respectively.
Then, for any square integrable function $f$, the following reconstruction formula
\[
    S[R[f;\rho]] = \iiprod{\sigma,\rho} f \quad \mbox{in} \quad L^2(\RR^m)\]
holds with the factor being a scalar product of $\sigma$ and $\rho$,
\[
\iiprod{\sigma,\rho} \coloneq  \int_\RR \sigma^\sharp(\omega)\overline{\rho^\sharp(\omega)}\svert \omega\svert ^{-m}\dd\omega,\]
where $\sharp$ denotes the Fourier transform.
\end{theorem}

From the perspective of neural network theory, 
the reconstruction formula claims a detailed/constructive version of the universal approximation theorem.
That is, 
given any target function $f$, 
as long as $\iiprod{\sigma,\rho} \neq 0$,
the network $S[\gamma]$ with coefficient $\gamma = R[f;\rho]$ reproduces the original function, and the coefficient is given explicit.

From the perspective of functional analysis, on the other hand,
the reconstruction formula states that $R$ and $S$ are analysis and synthesis operators, and thus play the same roles as, for instance, the Fourier ($F$) and inverse Fourier ($F^{-1}$) transforms respectively, in the sense that the reconstruction formula $S[R[f;\rho]]=\iiprod{\sigma,\rho}f$ corresponds to the Fourier inversion formula $F^{-1}[F[f]]=f$. 

Despite the common belief that neural network parameters are a blackbox, the
closed-form expression \nxrefeq{ridgelet} of ridgelet transform clearly describes how the network
parameters are distributed, which is a clear advantage of the integral
representation theory~\citep[see e.g.][]{Sonoda2021aistats}. 
Moreover, the integral representation theory can deal with a wide range of
activation functions without approximation, not only ReLU but all the tempered
distribution $\calS'(\RR)$~\citep[see e.g.][]{Sonoda2015acha}.

The ridgelet transform is discovered in the late 1990s independently
by~\citet{Murata1996} and~\citet{Candes.PhD}. 
The term ``ridgelet'' is named by Cand\`es, based on the facts that the graph of a function $\xx\mapsto\rho(\aa\cdot\xx-b)$ is ridge-shaped, and that the integral transform $R$ can be regarded as a multidimensional counterpart of the wavelet transform. 

In fact, the ridgelet transform can be decomposed into the composite of wavelet transform after the Radon transform, namely,
\begin{align*}
    R[f;\rho](a\uu,b) &= \int_{\RR} P[f](\uu,t)\overline{\rho(at-b)}\dd a \dd b, \quad (a,\uu,b) \in \RR\times\SS^{m-1}\times\RR, \\
    P[f](\uu,t) &\coloneq  \int_{(\RR \uu)^\perp} f(t\uu+\yy) \dd \yy, \quad (\uu,t) \in \SS^{m-1}\times\RR,
\end{align*}
where $\SS^{m-1}$ denotes the $m$-dimensional unit sphere, $(\RR\uu)^\perp \cong \RR^{m-1}$
denotes the orthocomplement of the normal vector $\uu \in \SS^{m-1}$, $\dd\yy$
denotes the Hausdorff measure on $(\RR\uu)^\perp$ or the Lebesgue measure on
$\RR^{m-1}$, and $\aa\in\RR^m$ is represented in polar coordinates $\aa = a \uu$
with $(a,\uu) \in \RR \times \SS^{m-1}$ allowing the double covering: $(a,\uu) = (-a,-\uu)$~\citep[see][for the proof]{Sonoda2015acha}.
Therefore, several authors have remarked that \emph{ridgelet analysis is
wavelet analysis in the Radon
domain}~\citep{Donoho2002,Kostadinova2014,Starck2010}.

In the context of deep learning
theory,~\citet{Savarese2019,Ongie2020,Parhi2021} and \citet{Unser2019} investigate the
ridgelet transform for the specific case of fully-connected ReLU layers to
establish the representer theorem. 
\citet{Sonoda2021aistats} have shown that the parameter distribution of a
finite model trained by regularized empirical risk minimization (RERM)
converges to the ridgelet spectrum $R[f;\rho]$ in an over-parametrized regime,
meaning that we can understand the parameters at local minima to be a finite
approximation of the ridgelet transform. 
In other words, analyzing neural network parameters can be turned to analyzing the ridgelet transform.

\subsection{Scope and Contribution of This Study}

On the other hand, one of the major shortcomings of ridgelet analysis is that
the closed-form expression is known for relatively small class of networks.
Indeed, until~\citet{sonoda2022symmetric,Sonoda2022gconv}, it was known only
for depth-2 fully-connected layer: $\sigma(\aa\cdot\xx-b)$. In the age of deep learning, a
variety of layers have become popular such as the convolution and pooling
layers~\citep{Fukushima1980,LeCun1998,Ranzato2006,Krizhevsky2012}. Furthermore,
the fully-connected layers on manifolds have also been developed such as the
hyperbolic network~\citep{Ganea2018hnn,Shimizu2021}.
Since the conventional ridgelet transform was discovered heuristically in the 1990s, and the derivation heavily depends on the specific structure of affine map $\aa\cdot\xx-b$, the ridgelet transforms for those modern architectures have been unknown for a long time.

In this study, we explain a systematic method to find the ridgelet transforms via the \emph{Fourier expression} of neural networks, and obtain \emph{new ridgelet transforms} in a unified manner.
The Fourier expression of $S[\gamma]$ is essentially a change-of-frame from neurons $\sigma(\aa\cdot\xx-b)$ to plane waves (or harmonic oscillators) $\exp(i\xx\cdot\xxi)$. Since the Fourier transform is extensively developed on a variety of domains, once a network $S[\gamma]$ is translated into a Fourier expression, we can systematically find a particular coefficient $\gamma_f$ satisfying $S[\gamma_f]=f$ via the Fourier inversion formula. In fact, the traditional ridgelet transform is re-discovered. 
Associated with the change-of-frame in $S[\gamma]$, the ridgelet transform
$R[f]$ is also given a Fourier expression, but this form is known as the
\emph{Fourier slice theorem} of ridgelet transform $R[f]$~\citep[see e.g.][]{Kostadinova2014}. Hence, we call our proposed method as the
\emph{Fourier slice method}.

Besides the classical networks, we deal with \emph{four} types of networks: 
\begin{enumerate}
    \item Networks on finite fields $\FF_p$ in \nxrefsec{case.finite},
    \item Group convolution networks on Hilbert spaces $\calH$ in \nxrefsec{case.gconv},
    \item Fully-connected networks on noncompact symmetric spaces $G/K$ in \nxrefsec{case.mfd}, and 
    \item Pooling layers (also known as the $d$-\emph{plane ridgelet transform}) in \nxrefsec{case.affine}.
\end{enumerate}
The first three cases are already published thus we only showcase them, while the last case (pooling layer and $d$-plane ridgelet) involves \emph{new} results.

For all the cases, the reconstruction formula $S[R[f]]=f$ is understood as a constructive proof of the \emph{universal approximation theorem} for corresponding networks.
The group convolution layer case widely extends the ordinary convolution layer
with periodic boundary, which is also the main subject of the so-called
\emph{geometric deep learning}~\citep{Bronstein2021}.
The case of fully-connected layer on symmetric spaces widely extends the
recently emerging concept of \emph{hyperbolic
networks}~\citep{Ganea2018hnn,Gulcehre2019,Shimizu2021}, which can be cast as
another geometric deep learning. 
The pooling layer case includes the original fully-connected layer and the
pooling layer; and the corresponding ridgelet transforms include previously
developed formulas such as the Radon transform formula by~\citet{Savarese2019}
and related to the previously developed ``$d$-plane ridgelet
transforms'' by~\citet{Rubin.ridgelet} and~\citet{Donoho.ridgelet}.

\subsection{General Notations}
 
For any integer $d>0$, $\calS(\RR^d)$ and $\calS'(\RR^d)$ denote the classes of Schwartz test functions (or rapidly decreasing functions) and tempered distributions on $\RR^d$, respectively. Namely, $\calS'$ is the topological dual of $\calS$.  We note that $\calS'(\RR)$ includes truncated power functions $\sigma(b)=b_+^k = \max\{b,0\}^k$ such as step function for $k=0$ and ReLU for $k=1$. 

\paragraph*{Fourier Transform}

To avoid potential confusion, we use two symbols $\widehat{\cdot}$ and $\cdot^\sharp$ for the Fourier transforms in $\xx \in \RR^m$ and $b \in \RR$, respectively. For example,
\begin{align*}
&\widehat{f}(\xxi) \coloneq  \int_{\RR^m}f(\xx)e^{-i\xx\cdot\xxi}\dd\xx, \quad \xxi \in \RR^m\\
   &\rho^\sharp(\omega) \coloneq  \int_\RR \rho(b)e^{-ib\omega}\dd b, \quad \omega \in \RR\\
   &\gamma^\sharp(\aa,\omega) = \int_\RR \gamma(\aa,b)e^{-ib\omega}\dd b, \quad (\aa,\omega) \in \RR^m\times\RR.
\end{align*}
Furthermore, with a slight abuse of notation, when $\sigma$ is a tempered
distribution (i.e.,~$\sigma \in \calS'(\RR)$), then $\sigma^\sharp$ is understood as the
Fourier transform of distributions. Namely, $\sigma^\sharp$ is another tempered
distribution satisfying $\int_\RR \sigma^\sharp(\omega)\phi(\omega)\dd\omega = \int_\RR \sigma(\omega) \phi^\sharp(\omega)\dd\omega$ for any test function $\phi \in \calS(\RR)$. We
refer to~\citet{Grafakos.classic} for more details on the Fourier transform for
distributions.

 \section{Method}
 \xlabel{sec:method}
We explain the basic steps to find the parameter distribution $\gamma$ satisfying $S[\gamma]=f$. The basic steps is three-fold: (\step{1}) Turn the network into the \emph{Fourier expression}, (\step{2}) \emph{change variables} inside the feature map into principal and auxiliary variables, and (\step{3}) put unknown $\gamma$ in the \emph{separation-of-variables form} to find a particular solution.
In the following, we conduct the basic steps for the classical setting, i.e.,~the case of the fully-connected layer, for the explanation purpose. However, the ``catch'' of this procedure is that it is applicable to a wide range of networks as we will see in the subsequent sections.

\subsection{Basic Steps to Solve $S[\gamma]=f$}

The following procedure is valid, for example, when $\sigma \in \calS'(\RR), \rho \in \calS(\RR), f \in L^2(\RR^m)$ and
$\gamma \in L^2(\RR^m\times\RR)$. See~\citet{Kostadinova2014} and~\citet{Sonoda2015acha} for more
details on the valid combinations of function classes.
\paragraph*{\step{1}.}
Using the convolution in $b$, we can turn the network into the \emph{Fourier expression} as below.
\begin{align*}
    S[\gamma](\xx)
    &\coloneq  \int_{\RR^m\times\RR} \gamma(\aa,b) \sigma( \aa\cdot\xx-b )\dd\aa\dd b \notag \\
    &= \int_{\RR^m} [\gamma(\aa,\cdot) *_b \sigma](\aa\cdot\xx)\dd\aa \notag \\
    &= \frac{1}{2\pi} \int_{\RR^m\times\RR} \gamma^\sharp(\aa,\omega) \sigma^\sharp( \omega ) e^{i\omega \aa\cdot \xx}\dd\aa\dd \omega.
\end{align*}
    Here, $*_b$ denotes the convolution in $b$; and the third equation follows from an identity (Fourier inversion formula) $\phi(b) = \frac{1}{2\pi}\int_\RR \phi^\sharp(\omega)e^{i\omega b}\dd\omega$ with $\phi(b) = [\gamma(\aa,)*_b\sigma](b)$ and $b = \aa\cdot\xx$.
\paragraph*{\step{2}.} Change variables $(\aa,\omega) = (\xxi/\omega,\omega)$ with $\dd \aa \dd \omega = \svert \omega\svert ^{-m} \dd \xxi\dd\omega$ so that feature map $\sigma^\sharp(\omega)e^{i\omega\aa\cdot\xx}$ splits into a product of a principal factor (in $\xxi$ and $\xx$) and an auxiliary factor (in $\omega$), namely
\[
    S[\gamma](\xx)
    = (2\pi)^{m-1} \int_\RR \left[ \frac{1}{(2\pi)^m} \int_{\RR^m} \gamma^\sharp(\xxi/\omega,\omega) e^{i\xxi\cdot \xx}\dd\xxi \right] \sigma^\sharp( \omega ) \svert \omega\svert ^{-m} \dd \omega.
\]
Now, we can see that the integration inside brackets $[\cdots]$ becomes the Fourier inversion with respect to $\xxi$ and $\xx$.
\paragraph*{\step{3}.} Because of the Fourier inversion, it is natural to assume that the unknown function $\gamma$ has a \emph{separation-of-variables form} as
\begin{equation}
    \gamma_{f,\rho}^\sharp(\xxi/\omega, \omega) \coloneq  \widehat{f}(\xxi) \overline{\rho^\sharp(\omega)}, \xlabel{eq:sep.var}
\end{equation}
with using an arbitrary function $\rho \in \calS(\RR)$.
Namely, it is composed of a principal factor $\widehat{f}$ containing the target function $f$, and an auxiliary factor $\rho^\sharp$ set only for convergence of the integration in $\omega$.
Then, we have
\begin{align*}
    S[\gamma_{f,\rho}](\xx)
    &= (2\pi)^{m-1} \int_\RR \left[ \frac{1}{(2\pi)^m} \int_{\RR^m} \widehat{f}(\xxi) e^{i\xxi\cdot \xx}\dd\xxi \right] \sigma^\sharp( \omega ) \overline{\rho^\sharp(\omega)}  \svert \omega\svert ^{-m} \dd \omega \notag \\
    &= \iiprod{\sigma,\rho} \frac{1}{(2\pi)^m}\int_{\RR^m} \widehat{f}(\xxi) e^{i\xxi\cdot \xx}\dd\xxi\dd \omega \notag \\
    &= \iiprod{\sigma,\rho} f(\xx),
\end{align*}
where we put
\[
    \iiprod{\sigma,\rho} \coloneq  (2\pi)^{m-1}\int_\RR \sigma^\sharp(\omega)\overline{\rho^\sharp(\omega)}\svert \omega\svert ^{-m}\dd\omega.
\]
In other words, the separation-of-variables expression $\gamma_{f,\rho}$ is a particular solution to the integral equation $S[\gamma]=cf$ with factor $c = \iiprod{\sigma,\rho} \in \CC$.

Furthermore, $\gamma_{f,\rho}$ is the ridgelet transform because it is rewritten as
\[
    \gamma_{f,\rho}^\sharp(\aa,\omega) = \widehat{f}(\omega\aa) \overline{\rho^\sharp(\omega)},\]
and thus calculated as
\begin{align*}
    \gamma(\aa,b)
    &= \frac{1}{2\pi}\int_{\RR} \widehat{f}(\omega\aa) \overline{\rho^\sharp(\omega) e^{-i\omega b}}\dd\omega \notag \\
    &= \frac{1}{2\pi}\int_{\RR^m \times \RR} f(\xx) \overline{\rho^\sharp(\omega) e^{i \omega (\aa\cdot\xx - b)}}\dd\xx\dd\omega \notag \\
    &= \int_{\RR^m \times \RR} f(\xx) \overline{\rho(\aa\cdot\xx-b)} \dd\xx,
\end{align*}
which is exactly the definition of the ridgelet transform $R[f;\rho]$.

In conclusion, the separation-of-variables expression \nxrefeq{sep.var} is the way to naturally find the ridgelet transform.
We note that Steps~1 and 2 can be understood as the \emph{change-of-frame} from the frame spanned by neurons $\{ \xx \mapsto \sigma(\aa\cdot\xx-b) \mid (\aa,b) \in \RR^m\times\RR\}$, with which we are less familiar, to the frame spanned by (the tensor product of an auxiliary function and the) plane wave $\{ \xx \mapsto \sigma^\sharp(\omega)e^{i\xxi\cdot\xx} \mid (\xxi,\omega) \in \RR^m\times\RR \}$, with which we are much familiar. Hence, in particular, the map $\gamma(\aa,b) \mapsto \gamma^\sharp(\aa/\omega,\omega)\svert \omega\svert ^{-m}$ can be understood as the associated coordinate transformation.

 \section{Case I: {{NN}} on Finite Field $\FF_p \coloneq  \ZZ/\ZZ_p$}
 \xlabel{sec:case.finite}
As one of the simplest applications, we showcase the results
by~\citet{Yamasaki2023icml}, a neural network on the finite field $\FF_p \coloneq  \ZZ/p\ZZ \cong \{ 0, 1, \ldots, p-1 \mod p\}$
(with prime number $p$). 
This study aimed to design a quantum algorithm that efficiently computes the ridgelet transform, and the authors developed this example based on the demand to represent all data and parameters in finite qubits. To be precise, the authors dealt with functions on discrete space $\FF_p^m$, as a discretization of functions on a continuous space $\RR^m$.

\subsection{Fourier Transform}

For any positive integers $n,m$, let $\ZZ_n^m \coloneq  (\ZZ/n\ZZ)^m$ denote the product of cyclic groups. We note that the set of all real functions $f$ on $\ZZ_n^m$ is identified with the ($n \times m$)-dimensional real vector space, i.e.~$\{ f : \ZZ_n^m \to \RR \} \cong \RR^{n \times m}$, because each value $f(i,j)$ of function $f$ at $(i,j) \in \ZZ_n^m$ can be identified with the $(i,j)$th component $a_{ij}$ of vector $\aa=(a_{ij}) \in \RR^{n \times m}$. In particular, $L^2(\ZZ_n^m) = \RR^{n \times m}$.

We use the Fourier transform on a product of cyclic groups as below.
\begin{definition}[Fourier Transform on {$\ZZ_n^m$}]
For any $f \in L^2(\ZZ_n^m)$, put
\[
    \widehat{f}(\xxi) \coloneq  \sum_{\xx \in \ZZ_n^m} f(\xx) e^{-2 \pi i \xxi \cdot \xx/n}, \quad \xxi \in \ZZ_n^m.
\]
\end{definition}

\begin{theorem}[Inversion Formula]
For any $f \in L^2(\ZZ_n^m)$,
\[
    f(\xx) = \frac{1}{\svert \ZZ_n^m\svert }\sum_{\xxi \in \ZZ_n^m} \widehat{f}(\xxi) e^{2 \pi i \xxi \cdot \xx/n}, \quad \xx \in \ZZ_n^m.
\]
\end{theorem}
The proof is immediate from the so-called \emph{orthogonality of characters}, an identity $\sum_{g \in \ZZ_n} e^{2 \pi i g(t-s)/n} = \svert \ZZ_n\svert  \delta_{ts} \ (t,s \in \ZZ_n)$, where $\delta_{ts}$ being the Kronecker's $\delta$.

We note that despite the Fourier transform itself can be defined on any cyclic group $\ZZ_n$, namely $n$ needs not be prime, a finite field $\FF_p (= \ZZ_p)$ is assumed to perform the change-of-variables step.

\subsection{Network Design}

Remarkably, the $\FF_p$-version of arguments is obtained by formally replacing every integration in the $\RR$-version of arguments with summation.
\begin{definition}[NN on {$\FF_p^m$}]
For any functions $\gamma \in L^2( \FF_p^m \times \FF_p )$ and $\sigma \in L^\infty(\FF_p)$, put
\[
        S[\gamma](\xx) \coloneq  \sum_{(\aa,b) \in \FF_p^m \times \FF_p} \gamma(\aa,b) \sigma(\aa\cdot\xx - b), \quad \xx \in \FF_p^m
\]
\end{definition}

Again, in~\citet{Yamasaki2023icml}, it is introduced as a discretized version
of a function on a continuous space $\RR^m$.
\subsection{Ridgelet Transform}

\begin{theorem}[Reconstruction Formula]
For any function $\rho \in L^\infty(\FF_p)$, put
\begin{align*}
    R[f;\rho](\aa,b) &\coloneq  \sum_{\xx \in \FF_p^m} f(\xx) \overline{\rho(\aa\cdot\xx-b)}, \quad (\aa,b) \in \FF_p^m \times \FF_p\\
    \iiprod{\sigma,\rho} &\coloneq  \frac{1}{\svert \FF_p\svert ^{m-1}}\sum_{\omega \in \FF_p} \sigma^\sharp(\omega) \overline{\rho^\sharp(\omega)}.
\end{align*}    
Then, for any $f \in L^2(\FF_p^m)$,
\[
        S[R[f;\rho]] = \iiprod{\sigma,\rho} f.
\]
\end{theorem}
In other words, the fully-connected network on finite field $\FF_p^m$ can strictly represent any square integrable function on $\FF_p^m$.

Finally, the following proof shows that a new example of neural networks can be obtained by systematically following the same three steps as in the original arguments.

\begin{proof}
{\textit{\step{1}}. Turn to the Fourier expression:}
\begin{align*}
    S[\gamma](\xx)
        &\coloneq  \sum_{(\aa,b) \in \FF_p^m \times \FF_p} \gamma(\aa,b) \sigma(\aa\cdot\xx - b)\\
        &= \frac{1}{\svert \FF_p\svert }\sum_{(\aa,\omega) \in \FF_p^m \times \FF_p} \gamma^\sharp(\aa,\omega) \sigma(\omega) e^{2 \pi i \omega \aa\cdot\xx/p}
        \end{align*}
{\textit{\step{2}}. Change variables $\xxi = \omega \aa$}
\begin{equation*}
        \phantom{S[\gamma](\xx)}= \frac{1}{\svert \FF_p\svert }\sum_{(\xxi,\omega) \in \FF_p^m \times \FF_p} \gamma^\sharp(\xxi/\omega,\omega) \sigma(\omega) e^{2 \pi i \xxi\cdot\xx/p}
        \end{equation*}
{\textit{\step{3}}. Put separation-of-variables form         $\gamma^\sharp(\xxi/\omega,\omega) = \widehat{f}(\xxi)\overline{\rho^\sharp(\omega)}$}
\begin{align*}
        &\phantom{S[\gamma](\xx)}=  \left( \svert \FF_p\svert ^{m-1}\sum_{\omega \in \FF_p} \sigma^\sharp(\omega) \overline{\rho^\sharp(\omega)} \right) \left(\frac{1}{\svert \FF_p\svert ^m} \sum_{\xxi \in \FF_p^m} \widehat{f}(\xxi)e^{2 \pi i \xxi\cdot\xx/p} \right)\\
        &\phantom{S[\gamma](\xx)}= \iiprod{\sigma,\rho} f(\xx),
\end{align*}
and we can verify $\gamma = R[f;\rho]$.
\end{proof}

 \section{Case {II}: {Group} Convolutional {{NN}} on {Hilbert} Space $\calH$}
 \xlabel{sec:case.gconv}
Next, we showcase the results by~\citet{Sonoda2022gconv}.
Since there are various versions of convolutional neural networks (CNNs), their approximation properties (such as the universality) have been investigated individually depending on the network architecture. The method presented here defines the generalized group convolutional neural network (GCNN) that encompasses a wide range of CNNs, and provides a powerful result by unifying the expressivity analysis in a constructive and simple manner by using ridgelet transforms.

\subsection{Fourier Transform}

Since the input to CNNs is a signal (or a function), the Fourier transform corresponding to a \emph{naive} integral representation is the Fourier transform on the space of signals, which is typically an infinite-dimensional space $\calH$ of functions. Although the Fourier transform on the infinite-dimensional Hilbert space has been well developed in the probability theory, the mathematics tends to become excessively advanced for the expected results. One of the important ideas of this study is to induce the Fourier transform of $\RR^m$ in a \emph{finite}-dimensional subspace $\subH$ of $\calH$ instead of using the Fourier transform on the entire space $\calH$. To be precise, we simply take an $m$-dimensional orthonormal frame $F_m = \{ h_i \}_{i=1}^m$ of $\calH$, put the linear span $\subH \coloneq  \hull F_m = \{ \sum_{i=1}^m c_i h_i \mid c_i \in \RR \}$, and identify each element $f = \sum_{i=1} c_i h_i \in \subH \subset \calH$ with point $\cc = (c_1, \ldots, c_m) \in \RR^m$. Obviously, this embedding depends on the choice of $m$-frame $F_m$, yet drastically simplifies the main theory itself. 

\begin{definition}[Fourier Transform on a Hilbert Space {$\subH \subset \calH$}]
Let $\calH$ be a Hilbert space, $\subH \subset \calH$ be an $m$-dimensional subspace, and $\lambda$ be the Lebesgue measure induced from $\RR^m$. Put
\[
        \widehat{f}(\xi) \coloneq  \int_{\subH} f(x) e^{-i \iprod{x,\xi}} \dd \lambda(x), \quad \xi \in \subH
\]
\end{definition}

Then, obviously from the construction, we have the following.
\begin{theorem}
For any $f \in L^2(\subH)$,
\[\frac{1}{(2\pi)^m} \int_{\subH} \widehat{f}(\xi) e^{i \iprod{x,\xi}}  \dd\lambda(\xi) = f(x), \quad x \in \subH
\]    
\end{theorem}

\subsection{Network Design}

Another important idea is to deal with various group convolutions in a uniform manner by using the linear representation of groups.

\begin{definition}[Generalized Group Convolution]
Let $G$ be a group, $\calH$ be a Hilbert space, and $T:G\to GL(\calH)$ be a group representation of $G$. The $(G,T)$-convolution is given by
\[
    (a * x)(g) \coloneq  \iprod{ T_{g^{-1}}[x],a}_{\calH}, \quad a,x \in \calH.
\]

As clarified in~\citet{Sonoda2022gconv}, the \emph{generalized convolution}
 covers a variety of typical examples such as (1) classical group
 convolution $\int_G x(h) a(h^{-1}g) \dd h$, (2) discrete cyclic convolution for
 multi-channel digital images, (3) DeepSets, or permutation equivariant
 maps, (4) continuous cyclic convolution for signals, and (5)
 $\mathrm{E}(n)$-equivariant maps.
\end{definition}

\begin{definition}[Group CNN]
Let $\subH \subset \calH$ be an $m$-dimensional subspace equipped with the Lebesgue measure $\lambda$. Put
\[
    S[\gamma](x)(g) \coloneq  \int_{\subH \times \RR} \gamma(a,b) \sigma( (a * x)(g) - b ) \dd \lambda(a) \dd b, \quad x \in \calH, g \in G
\]
\end{definition}

Here, the integration runs all the possible convolution filters $a$. For the sake of simplicity, however, the domain $\subH$ of filters is restricted to an $m$-dimensional subspace of entire space $\calH$.

\subsection{Ridgelet Transform}

In the following, $e \in G$ denotes the identity element.

\begin{definition}[{$(G,T)$}-Equivariance]
A (nonlinear) map $f:\calH\to\CC^G$ is $(G,T)$-equivariant when
\[
    f( T_g[x] )(h) = f(x)(g^{-1}h), \quad \forall x \in \subH, g,h \in G
\]
\end{definition}

We note that the proposed network is inherently $(G,T)$-equivariant
\[
        S[\gamma]( T_g[x] )(h) = S[\gamma](x)(g^{-1}h), \quad \forall x \in \calH, \ g,h \in G.
\]

\begin{definition}[Ridgelet Transform]
For any measurable functions $f : \subH \to \CC^G$ and $\rho:\RR\to\CC$, put
\[
    R[f;\rho](a,b) \coloneq  \int_{\subH} f(x)(e) \overline{\rho( \iprod{a,x}_\calH - b )} \dd \lambda(x).
\]
\end{definition}
It is remarkable that the product of $a$ and $x$ inside the $\rho$ is not convolution $a * x$ but scalar product $\iprod{a,x}$. This is essentially because (1) $f$ will be assumed to be group equivariant, and (2) the network is group equivariant by definition.

\begin{theorem}[Reconstruction Formula]
Suppose that $f$ is $(G,T)$-equivariant and $f(\bullet)(e) \in L^2(\subH)$, then $S[R[f;\rho]]=\iiprod{\sigma,\rho} f$.
\end{theorem}
In other words, a continuous GCNN can represent any square-integrable \emph{group-equivariant} function.
Again, the proof is performed by systematically following the three steps as below.

\begin{proof}
{\textit{\step{1}}. Turn to a Fourier expression:}
\begin{align*}
S[\gamma](x)(g)
        &= \int_{\subH\times\RR} \gamma(a,b) \sigma( \iprod{T_{g^{-1}}[x],a}_{\calH} - b) \dd a \dd b \\
        &=\frac{1}{2\pi}\int_{\subH\times\RR} \gamma^\sharp(a,\omega) \sigma^\sharp(\omega)e^{i\omega\iprod{T_{g^{-1}}[x],a}_{\calH}}\dd a \dd\omega 
    \end{align*}
{\textit{\step{2}}. Change variables $(a,\omega) = (\xi/\omega,\omega)$ with $\dd a \dd \omega = \svert \omega\svert ^{-m}\dd\xi\dd\omega$:}
\begin{equation*}
        \phantom{S[\gamma](x)(g)}= \frac{1}{2\pi}\int_{\subH\times\RR} \gamma^\sharp(\xi/\omega,\omega)\sigma^\sharp(\omega)e^{i\iprod{T_{g^{-1}}[x],\xi}_{\calH}} \svert \omega\svert ^{-m} \dd\xi\dd\omega.
\end{equation*}
{\textit{\step{3}}. Put separation-of-variables form
    $\gamma_{f,\rho}^\sharp(\xi/\omega,\omega) \coloneq  \widehat{f}(\xi)(e) \overline{\rho^\sharp(\omega)}$}
\begin{align*}
        &\phantom{S[\gamma](x)(g)}= \frac{1}{2\pi}\int_{\subH} \widehat{f}(\xi)(e) e^{i\iprod{T_{g^{-1}}[x],\xi}_{\calH}} \dd\lambda(\xi) \int_\RR \sigma^\sharp(\omega) \overline{\rho^\sharp(\omega)}\svert \omega\svert ^{-m} \dd\omega \notag \\
        &\phantom{S[\gamma](x)(g)}= \iiprod{\sigma,\rho} f( x )(g),
\end{align*}
and we can verify $\gamma_{f,\rho}=R[f;\rho]$.
\end{proof}

\subsection{Literature in Geometric Deep Learning}

General convolution networks for geometric/algebraic domains have been
developed for capturing the invariance/equivariance to the symmetry in a
data-efficient
manner~\citep{Bruna2013,Cohen2016a,Zaheer2017,Kondor2018,Cohen2019,Kumagai2020}.
To this date, quite a variety of convolution networks have been proposed for
grids, finite sets, graphs, groups, homogeneous spaces and manifolds. We refer
to~\citet{Bronstein2021} for a detailed survey on the so-called \emph{geometric
deep learning}.

Since a universal approximation theorem (UAT) is a corollary of a reconstruction formula, $S[R[f;\rho]]=\iiprod{\sigma,\rho}f$, the 3-steps Fourier expression method have provided a variety of UATs for $\sigma(ax-b)$-type networks in a \emph{unified manner}. 
Here, we remark that the UATs of individual convolution networks have already
shown~\citep{Maron2019universality,Zhou2020a,Yarotsky2021a}. However, in
addition to above mentioned advantages, the \emph{wide coverage} of activation
functions $\sigma$ is also another strong advantage. In particular, we do
not need to rely on the specific features of ReLU, nor need to rely on Taylor
expansions/density arguments/randomized assumptions to deal with nonlinear
activation functions.

 \section{Case {III}: {{NN}} on Noncompact Symmetric Space $X=G/K$}
 \xlabel{sec:case.mfd}
Then, we showcase the results by~\citet{sonoda2022symmetric}.
When the data is known to be on a certain manifold, it is natural to consider developing NNs on the manifold, in order to explicitly incorporate the inductive bias. Since there are no such thing as the standard inner products or affine mappings on manifolds, various NNs have been proposed based on geometric considerations and implementation constraints. The main idea of this study is to start from the Fourier transform on a manifold and induce a NN on the manifold and its reconstruction formula. On compact groups such as spheres $\SS^{m-1}$, the Fourier analysis is well known as the Peter--Weyl theorem, but in this study, the authors focused on noncompact symmetric spaces $G/K$ such as hyperbolic space $\HH^m$ and space $\PP_m$ of positive definite matrices and developed NNs based on the Helgason--Fourier transform on noncompact symmetric space.

\subsection{Noncompact Symmetric Space $G/K$}

    We refer to~\citet[Introduction]{Helgason.GGA}
and~\citet[Chapter~III]{Helgason.GASS}.
    A noncompact symmetric space is a homogeneous space $G/K$ with nonpositive sectional curvature on which $G$ acts transitively.
    Two important examples are hyperbolic space $\HH^m$ (\reffig{poincare}) and SPD manifold $\PP_m$. See also Appendices~\ref{sec:poincare} and \ref{sec:spdmfd} for more details on  these spaces.
\begin{figure}[t]
    \centering
    \includegraphics[width=.5\textwidth]{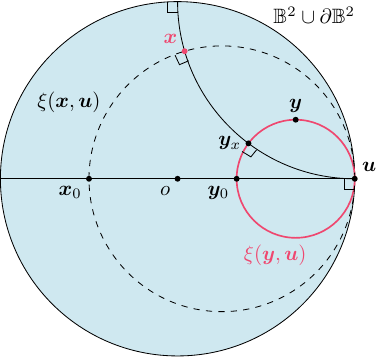}
\caption{Poincare Disk $\BB^2$ is a noncompact symmetric space $SU(1,1)/SO(2)$. Poincar\'e disk $\BB^2$, boundary $\bdB^2$, point $\xx$ (magenta), horocycle $\xi(\yy,\uu)$ (magenta) through point $\yy$ tangent to the boundary at $\uu$, and two geodesics (solid black)  orthogonal to the boundary at $\uu$ through $\oo$ and $\xx$ respectively. The signed composite distance $\iprod{\yy,\uu}$ from the origin $\oo$ to the horocycle $\xi(\yy,\uu)$ can be visualized as the Riemannian distance from $\oo$ to point $\yy_0$. Similarly, the distance between point $\xx$ and horocycle $\xi(\yy,\uu)$ is understood as the Riemannian distance between $\xx$ and $\yy_x$ along the geodesic, or equivalently, $\xx_0$ and $\yy_0$. See \nxrefapp{poincare} for more details.}
    \label{fig:poincare}
\end{figure}

    Let $G$ be a connected semisimple Lie group with finite center, and let $G=KAN$ be its Iwasawa decomposition. Namely, it is a unique diffeomorphic decomposition of $G$ into subgroups $K,A$, and $N$, where $K$ is maximal compact, $A$ is maximal abelian, and $N$ is maximal nilpotent.
    For example, when $G=GL(m,\RR)$ (general linear group), then $K=O(m)$ (orthogonal group), $A = D_+(m)$ (all positive diagonal matrices), and $N = T_1(m)$ (all upper triangular matrices with ones on the diagonal).

    Let $\dd g, \dd k, \dd a$, and $\dd n$ be left $G$-invariant measures on $G,K,A$, and $N$ respectively.
    
    Let $\lieG, \lieK,\lieA$, and $\lieN$ be the Lie algebras of $G,K,A$, and $N$ respectively.
    By a fundamental property of abelian Lie algebra, both $\lieA$ and its dual $\lieA^{*}$ are the same dimensional vector spaces, and thus they can be identified with $\RR^r$ for some $r$, namely $\lieA = \lieA^{*} = \RR^r$. We call $r \coloneq  \dim \lieA$ the rank of $X$. For example, when $G=GL(m,\RR)$, then $\lieG=\mathfrak{gl}_m = \RR^{m\times m}$ (all $m \times m$ real matrices), $\lieK = \mathfrak{o}_m$ (all skew-symmetric matrices), $\lieA=D(m)$ (all diagonal matrices), and $\lieN=T_0(m)$ (all strictly upper triangular matrices).

\begin{definition}
Let $X \coloneq  G/K$ be a noncompact symmetric space, namely, a Riemannian manifold composed of all the left cosets
\[
        X \coloneq  G/K \coloneq  \{ x=gK \mid g \in G \}.
\]
Using the identity element $e$ of $G$, let $o=eK$ be the origin of $X$. By the construction of $X$, group $G$ acts transitively on $X$, and let $g[x] \coloneq  ghK$ (for $x=hK$) denote the $G$-action of $g \in G$ on $X$. Specifically, any point $x\in X$ can always be written as $x=g[o]$ for some $g \in G$. Let $\dd x$ denote the left $G$-invariant measure on $X$. 
\end{definition}

\begin{example}[Hyperbolic Space {$\HH^m = SO^+(1,m)/O(m)$}]
It is used for embedding words and tree-structured dataset.
\end{example}

\begin{example}[SPD Manifold {$\PP_m = GL(m)/O(m)$}]
It is a manifold of positive definite matrices, such as covariance matrices.
\end{example}

\subsection{Boundary $\bdX$, Horosphere $\xi$, and Vector-Valued Composite Distance $\iprod{x,u}$}

    We further introduce three geometric objects in noncompact symmetric space $G/X$.
    In comparison to Euclidean space $\RR^m$,
    the boundary $\bdX$ corresponds to ``the set of all infinite points'' $\lim_{r \to +\infty} \{ r \uu \mid \svert \uu\svert =1, \uu \in \RR^m \}$,
    a horosphere $\xi$ through point $x \in X$ with normal $u \in \bdX$ corresponds to a straight line $\xxi$ through point $\xx \in \RR^m$ with normal $\uu \in \SS^{m-1}$,
    and the vector distance $\iprod{x,u}$ between origin $o \in X$ and horosphere $\xi(x,u)$ corresponds to the Riemannian distance between origin $\mathbf{0} \in \RR^m$ and straight line $\xxi(\xx,\uu)$.

\begin{definition}
Let $M \coloneq  C_K(A) \coloneq  \{ k \in K \mid ka = ak \mbox{ for all } a \in A \}$ be the centralizer of $A$ in $K$, and let 
\[
    \bdX \coloneq  K/M \coloneq  \{ u=kM \mid k \in K \}\]
be the boundary (or ideal sphere) of $X$, which is known to be a compact manifold. Let $\dd u$ denote the uniform probability measure on $\bdX$.
\end{definition}

For example, when $K=O(m)$ and $A=D_+(m)$, then $M = D_{\pm 1}$ (the subgroup of $K$ consisting of diagonal matrices with entries $\pm 1$).

\begin{definition}
Let
\[
        \Xi \coloneq  G/MN \coloneq  \{ \xi = gMN \mid g \in G\}\]
be the space of horospheres. 
\end{definition}

Here, basic horospheres are: An $N$-orbit $\xi_o \coloneq  N[o] = \{ n[o] \mid n \in N \}$, which is a horosphere passing through the origin $x=o$ with normal $u=eM$; and $ka[\xi_o] = kaN[o]$, which is a horosphere through point $x=ka[o]$ with normal $u=kM$. In fact, any horosphere can be represented as $\xi(kan[o], kM)$ since $kaN = kanN$ for any $n \in N$. We refer to~\citet[Ch.I, \textsection~1]{Helgason.GASS}
and~\citet[\textsection~3.5]{Bartolucci2021} for more details on the horospheres and
boundaries.

\begin{definition}
As a consequence of the Iwasawa decomposition, for any $g \in G$ there uniquely exists an $r$-dimensional vector $H(g) \in \lieA$ satisfying $g \in K e^{H(g)}N$. For any $(x,u) = (g[o], kM) \in X \times \bdX$, put 
\[
    \iprod{x,u} \coloneq  -H(g^{-1}k) \in \lieA \cong \RR^r,\]
which is understood as the $r$-dimensional vector-valued distance, called the \emph{composite distance}, from the origin $o\in X$ to the horosphere $\xi(x,u)$ through point $x$ with normal $u$.
\end{definition}

Here, the vector-valued distance means that the $\ell^2$-norm coincides with the Riemannian length, that is, $\svert \iprod{x,u}\svert  = \svert d( o, \xi(x,u) )\svert $.
    We refer to~\citet[Ch.II, \textsection~1, 4]{Helgason.GASS}
and~\citet[\textsection~2]{Kapovich2017} for more details on the vector-valued composite
distance.

\subsection{Fourier Transform}

Based on the preparations so far, we introduce the Fourier transform on $G/K$, known as the Helgason--Fourier transform.
Let $W$ be the Weyl group of $G/K$, and let $\svert W\svert $ denote its order. Let $\cc(\lambda)$ 
be the Harish-Chandra $\cc$-function for $G$. We refer
to~\citet[Theorem~6.14, Ch.~IV]{Helgason.GGA} for the closed-form expression of
the $\cc$-function.

\begin{definition}[Helgason--Fourier Transform]
For any measurable function $f$ on $X$, put
\[\widehat{f}(\lambda,u) \coloneq  \int_{X} f(x) e^{(-i\lambda+\varrho)\iprod{x,u}}\dd x, \ (\lambda,u) \in \lieA^{*} \times \bdX\]
with a certain constant vector $\varrho \in \lieA^{*}$.
Here, the exponent $(-i\lambda+\varrho)\iprod{x,u}$ is understood as the action of functional $-i\lambda+\varrho \in \lieA^{*}$ on a vector $\iprod{x,u} \in \lieA$.
\end{definition}
This is understood as a ``Fourier transform'' because $e^{(-i\lambda+\varrho)\iprod{x,u}}$ is the eigenfunction of Laplace--Beltrami operator $\Delta_X$ on $X$.

\begin{theorem}[Inversion Formula]
For any square integrable function $f \in L^2(X)$,
\[
    f(x) = \frac{1}{\svert W\svert } \int_{\lieA^{*} \times \bdX} \widehat{f}(\lambda,u) e^{(i\lambda+\varrho)\iprod{x,u}} \frac{\dd\lambda\dd u}{\svert \cc(\lambda)\svert ^{2}}, \quad x \in X.
\]
\end{theorem}
We refer to~\citet[Theorems~1.3 and 1.5, Ch.~III]{Helgason.GASS} for more
details on the inversion formula.

\subsection{Network Design}

In accordance with the geometric perspective, it is natural to define the network as below.
\begin{definition}[NN on Noncompact Symmetric Space {$G/K$}] 
For any measurable functions $\sigma:\RR\to\CC$ and $\gamma:\lieA^{*}\times\bdX\times\RR\to\CC$, put
\[
    S[\gamma](x)
    \coloneq  \int_{\lieA^{*}\times\bdX\times\RR} \gamma(a,u,b)\sigma( a \iprod{x,u} - b ) e^{\varrho \iprod{x,u}} \dd a \dd u \dd b, \quad x \in G/K.
\]
\end{definition}
Remarkably, the scalar product $\aa\cdot\xx$ (or $ a\uu\cdot\xx$ in polar coordinate) in the Euclidean setting is replaced with a distance function $a\iprod{x,u}$ in the manifold setting.
In~\citet{sonoda2022symmetric}, the authors instantiate two important examples
as below.

\begin{example}[Continuous Horospherical Hyperbolic NN]
On the \emph{Poincare ball model} $\BB^m \coloneq  \{ \xx \in \RR^m \mid \svert \xx\svert  < 1\}$ equipped with the Riemannian metric $\mathfrak{g} = 4(1-\svert \xx\svert )^{-2} \sum_{i=1}^m \dd x_i \otimes \dd x_i$,
\[            S[\gamma](\xx) \coloneq  \int_{\RR\times\bdB^m\times\RR}\gamma(a,\uu,b)\sigma(a \iprod{\xx,\uu}-b)e^{\varrho \iprod{\xx,\uu}}\dd a\dd\uu\dd b,  \quad \xx \in \BB^m
\]
            $\varrho = (m-1)/2$, 
            $\iprod{\xx,\uu} = \log\left( \frac{1-\svert \xx\svert _E^2}{\svert \xx-\uu\svert _E^2} \right), \quad (\xx,\uu) \in \BB^m \times \bdB^m$
\end{example}
\begin{example}[Continuous Horospherical SPD Net]
On the SPD manifold $\PP_m$,
\[
S[\gamma](x) \coloneq  \int_{\RR^m\times\bdP_m\times\RR} \gamma(\aa,u,b) \sigma(\aa \cdot \iprod{x,u} -b) e^{\boldsymbol{\varrho} \cdot \iprod{x,u}} \dd \aa \dd u \dd b, \quad x \in \PP_m
\]
            $\boldsymbol{\varrho} = (-\frac{1}{2}, \ldots, -\frac{1}{2}, \frac{m-1}{4})$,
            $ \iprod{\xx,\uu} = \frac{1}{2}\log \lambda\left(u x u^\top \right), \quad (x,u) \in \PP_m \times \bdP_m$ where $\lambda(x)$ denotes the diagonal in the \emph{Cholesky decomposition} of $x$.
\end{example}

\subsection{Ridgelet Transform}

\begin{definition}[Ridgelet Transform]
For any measurable functions $f:X \to \CC$ and $\rho:\RR\to\CC$, put
\begin{align*}
&R[f;\rho](a,u,b) \coloneq  \int_X \cc[f](x)\overline{\rho(a\iprod{x,u}-b)}e^{\varrho\iprod{x,u}}\dd x,\\
    &\cc[f](x) \coloneq  \int_{\lieA^{*}\times\bdX}\widehat{f}(\lambda,u)e^{(i\lambda+\varrho)\iprod{x,u}}\frac{\dd\lambda\dd u}{\svert W\svert\,\svert \cc(\lambda)\svert ^{4}},\\
    &\iiprod{\sigma,\rho} \coloneq  \frac{\svert W\svert }{2\pi}\int_\RR \sigma^\sharp(\omega)\overline{\rho^\sharp(\omega)}\svert \omega\svert ^{-r}\dd\omega.
\end{align*}
Here $\cc[f]$ is defined as a multiplier satisfying $\widehat{\cc[f]}(\lambda,u)=\widehat{f}(\lambda,u)\svert \cc(\lambda)\svert ^{-2}$.
\end{definition}

\begin{theorem}[Reconstruction Formula] 
Let $\sigma \in \calS'(\RR), \rho \in \calS(\RR)$. 
Then, for any square integrable function $f$ on $X$, we have 
\[
    S[R[f;\rho]](x)
    = \int_{\lieA^{*}\times\bdX\times\RR} R[f;\rho](a,u,b)\sigma(a\iprod{x,u}-b) 
    e^{\varrho\iprod{x,u}}\dd a \dd u \dd b
    = \iiprod{\sigma,\rho} f(x).
\]
\end{theorem}
In other words, the fully-connected network on noncompact symmetric space $G/K$ can represent any square-integrable function.
Again, the proof is performed by systematically following the three steps as below.

\begin{proof}
We identify the scale parameter $a \in \lieA^{*}$ with vector $\aa \in \RR^r$.
\par\noindent
{\textit{\step{1}}. Turn to a Fourier expression:}
\begin{align*}
    S[\gamma](x)
    &\coloneq  \int_{\RR^r\times\bdX\times\RR} \gamma(\aa,u,b)\sigma(\aa\cdot\iprod{x,u}-b)e^{\varrho\iprod{x,u}}\dd\aa\dd u\dd b \notag \\
    &= \frac{1}{2\pi}\int_{\RR^r\times\bdX\times\RR} \gamma^\sharp(\aa,u,\omega)\sigma^\sharp(\omega)e^{(i\omega\aa+\varrho)\iprod{x,u}}\dd\aa\dd u\dd\omega
    \end{align*}
{\textit{\step{2}}. Change variables $(\aa,\omega) = (\llambda/\omega,\omega)$ with $\dd\aa\dd\omega = \svert \omega\svert ^{-r}\dd\llambda\dd\omega$:}
\begin{align*}
    &\phantom{S[\gamma](x)}= \frac{1}{2\pi}\int_{\RR} \left[ \int_{\lieA^{*}\times\bdX} \gamma^\sharp(\lambda/\omega,u,\omega) e^{(i\lambda+\varrho)\iprod{x,u}} \dd\lambda\dd u\right] \sigma^\sharp(\omega)\svert \omega\svert ^{-r}\dd\omega.
\end{align*}
{\textit{\step{3}}. Put separation-of-variables form $\gamma_{f,\rho}^\sharp(\lambda/\omega,u,\omega) = \widehat{f}(\lambda,u)\overline{\rho^\sharp(\omega)}\svert \cc(\lambda)\svert ^{-2}$}
\begin{align*}
    &\phantom{S[\gamma](x)}= \left(\frac{\svert W\svert }{2\pi}\int_{\RR} \sigma^\sharp(\omega)\overline{\rho^\sharp(\omega)}\svert \omega\svert ^{-r}\dd\omega \right)\left(\int_{\lieA^{*}\times\bdX} \widehat{f}(\lambda,u) e^{(i\lambda+\varrho)\iprod{x,u}} \frac{\dd\lambda\dd u}{\svert W\svert\,\svert \cc(\lambda)\svert ^2} \right) \\
    &\phantom{S[\gamma](x)}= \iiprod{\sigma,\rho} f(x),
\end{align*}
and we can verify $\gamma_{f,\rho}=R[f;\rho]$.
\end{proof}

\subsection{Literature in Hyperbolic Neural Networks}

The hyperbolic neural network
(HNN)~\citep{Ganea2018hnn,Gulcehre2019,Shimizu2021} is another emerging
direction of geometric deep learning, inspired by the empirical observations
that some datasets having tree or hierarchical structure can be efficiently
embedded into hyperbolic
spaces~\citep{Krioukov2010,Nickel2017,Nickel2018,Sala2018}. 
We note that designing a FC layer $\sigma(\iprod{a,x}-b)$ on manifold $X$ is less trivial,
because neither \emph{scalar product} $\iprod{a,x}$, \emph{bias subtraction} $-b$, nor \emph{elementwise activation} of $\sigma$ is trivially defined on $X$ in general, and thus we have to face those primitive issues.

The design concept of the original HNN~\citep{Ganea2018hnn} is to reconstruct
basic operations in the ordinary neural networks such as linear maps, bias
translations, pointwise nonlinearities and softmax layers by using the
Gyrovector operations in a tractable and geometric manner.
For example, in HNN++ by~\citet{Shimizu2021}, the Poincar\'e multinomial
logistic regression (MLR) layer $v(\xx)$ and fully-connected (FC) layer
$\calF(\xx)$,  corresponding to the $1$-affine layer $\aa\cdot\xx-b$ and
$k$-affine layer $A^\top\xx-\bb$ without activation respectively, are
designed as nonlinear maps $v:\HH^m\to\RR$ and  $\calF:\HH^m\to\HH^n$ so that 
$v(\xx;\aa,b)$ coincides with the distance between output $\yy=\calF(\xx;\{\aa_i,b_i\}_{i\in[n]})$ and Poincar\'e hyperplane $H(o,\ee_{\aa,b})$. Here, a Poincar\'e hyperplane $H(\xx,\zz)$ is defined as the collection of all geodesics through point $\xx$ and normal to $\zz$.
Furthermore, they are designed so that the discriminative hyperplane coincides a Poincar\'e hyperplane.
The nonlinear activation function $\sigma:\RR^m\to\RR^n$ is cast as a map $\sigma:\HH^m\to\HH^k$ via lifting
    $\sigma(\xx) \coloneq  \exp_0 \circ \, \sigma  \circ \log_0(\xx)$ for any $\xx \in \HH^m$.
However, in practice, activation can be omitted since the FC layer is inherently nonlinear.

\begin{figure}[t]
    \centering
    \includegraphics[width=.5\textwidth]{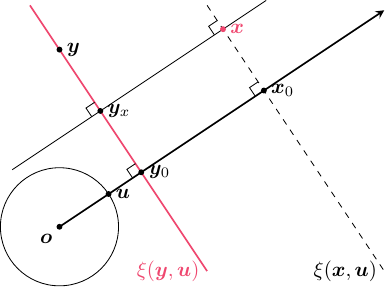} %
\caption{The Euclidean fully-connected layer $\sigma(\aa\cdot\xx-b)$ is recast as the signed distance $d(\xx,\xi)$ from a point $\xx$ to a hyperplane $\xi(\yy,\uu)$ followed by nonlinearity $\sigma(r\bullet)$, where $\yy$ satisfies $r\yy\cdot\uu=b$ and $\xi(\yy,\uu)$ passes through the point $\yy$ with normal $\uu$.\xlabel{fig:euclid}}
\end{figure}

In this study, on the other hand, we take $X$ to be a \emph{noncompact symmetric space} $G/K$, which is a generalized version of the hyperbolic space $\HH^m$. 
Following the philosophy of the Helgason--Fourier transform, 
we regard the scalar product $\uu\cdot\xx$ of unit vector $\uu \in \SS^{m-1}$ and point $\xx \in \RR^m$ as the signed distance between the origin $\oo$ and plane $\xi(\xx,\uu)$ through point $\xx$ with normal $\uu$. Then, we recast it to the vector-valued distance, denoted $\iprod{u,x}$, between the origin $o$ and horocycle $\xi(x,u)$ through point $x$ with normal $u$. As a result, we can naturally define bias subtraction $-b$ and elementwise activation of $\sigma:\RR\to\RR$ because the signed distance is identified with a vector.

More geometrically,
$\uu\cdot\xx-b$ in $\RR^m$ is understood as the distance between point $\xx$ and plane $\xi$ satisfying $\uu\cdot\xx-b=0$ (see \reffig{euclid}). Similarly, 
$\iprod{u,x}-b$ is understood as the distance between point $x$ and horocycle $\xi$ satisfying $\iprod{u,x}-b=0$.
Hence, as a general principle,
we may formulate a versatile template of affine layers on $X$ as
\begin{equation}
    S[\gamma](x) \coloneq  \int_{\RR\times\Xi} \gamma(a,\xi) \sigma( a d(x,\xi) ) \dd a \dd \xi.
\end{equation}
For example, in the original HNN, the Poincar\'e hyperplane $H$ is employed as the geometric object. If we have a nice coordinates such as $(s,t) \in \RR^m \times \RR^m$ satisfying $d(x(t),\xi(s)) = t-s$, then we can turn it to the Fourier expression and hopefully obtain the ridgelet transform.

The strengths of our results are summarized as that we obtained the ridgelet transform in a unified manner for a wide class of input domain $X$ in a geometric manner, i.e.,~independent of the coordinates; in particular, that it is the first result to define the neural network and obtained the ridgelet transform on noncompact space.

 \section{Case {IV}: {Pooling} Layer and $d$-plane Ridgelet Transform}
 \xlabel{sec:case.affine}
Finally, we present several new results.
Technically, we consider networks with \emph{multi}variate activation functions $\sigma:\RR^k \to \CC$. In all the sections up to this point, we have considered \emph{uni}variate activation function $\sigma:\RR\to\CC$ (i.e.~$k=1$). In the context of neural networks, it is understood as a mathematical model of pooling layers such as
\begin{alignat*}
{3}
    \sigma(\bb) &= \frac{1}{k}\sum_{i=1}^k b_i & \quad & \mbox{(average pooling),}\\
    \sigma(\bb) &= \max_{i \in [k]}\{ b_i \} & & \mbox{(max pooling), and} \\
    \sigma(\bb) &= \svert  \bb \svert _p & & \mbox{($\ell^p$-norm).}
\end{alignat*}
Meanwhile, in the context of sparse signal processing in the 2000s such
as~\citet{Donoho.ridgelet} and~\citet{Rubin.ridgelet} (see also
\nxrefsec{literature2000}), it can also be understood as the ridgelet transform
corresponding to the so-called $d$-plane transform (see also
\nxrefsec{dplane}).

As mentioned in \nxrefsec{ridgelet.classic},
the ridgelet transforms have profound relations to the \emph{Radon} and wavelet transforms. %
In the language of probability theory, a Radon transform is understood as a \emph{marginalization}, and the traditional problem of Johann Radon is the inverse problem of reconstructing the original joint distribution from several marginal distributions.
Hence, depending on the choice of variables to be marginalized, there are countless different Radon transforms. In other words, the Radon transform can also be a rich source for finding a variety of novel networks and ridgelet transforms. In this section, we derive the ridgelet transform corresponding to the $d$-plane transform. (Nonetheless, the proofs are shown by the 3-steps Fourier expression method.)

\subsection*{Additional Notations}
 \xlabel{sec:dplane.notation}
Let $m,d,k$ be positive integers satisfying $m=d+k$;
let $M_{m,k} \coloneq  \{ A \in \RR^{m\times k} \mid \rank A = k \}$ be a set of all full-column-rank (i.e.,~injective) matrices equipped with the Lebesgue measure $\dd A = \bigwedge_{ij} \dd a_{ij}$;
let $V_{m,k} \coloneq  \{ U = [\uu_1,\ldots,\uu_k] \in \RR^{m\times k} \mid U^\top U=I_k \}$ be the Stiefel manifold of orthonormal $k$-frames in $\RR^m$ equipped with invariant measure $\dd U$;
let $O(k) \coloneq  \{ V \in \RR^k \mid V^\top V = I_k \}$ be the orthogonal group in $\RR^k$ equipped with invariant measure $\dd V$.
In addition, let $GV_{m,k} \coloneq  \{ A = a U \in \RR^{m\times k} \mid a \in \RR_+, U \in V_{m,k} \}$ be a similitude group equipped with the product measure $\dd a \dd U$. 
For a rectangular matrix $A \in M_{m,k}$, we write $\svert \det A\svert  \coloneq  \svert \det A^\top A\svert ^{1/2}$ for short. In the following, we use $\widehat{\cdot}$ and $\cdot^\sharp$ for the Fourier transforms in $\xx\in\RR^m$ and $\bb\in\RR^k$, respectively.
For any $s \in \RR$, let $\triangle^{s/2}$ denote the fractional Laplacian defined as a Fourier multiplier: 
$\triangle^{s/2}[f](\xx) \coloneq  \frac{1}{(2\pi)^m}\int_{\RR^m} \svert \xxi\svert ^{s} \widehat{f}(\xxi) e^{i\xxi\cdot\xx} \dd\xxi$.

\subsection{$d$-plane transform}
\xlabel{sec:dplane}

The $d$-plane transform is a Radon transform that marginalizes a $d$-dimensional affine subspace ($d$-plane) in an $m$-dimensional space. In the special cases when $d = m-1$ (hyperplane) and $d = 1$ (straight line), they are respectively called the (strict) Radon transform and the X-ray transform. 
The ridgelet transforms to be introduced in this section correspond to $d$-plane Radon transform, and the classical ridgelet transform corresponds to the strict Radon transform ($d=m-1$).
We refer to Chapter~1 of~\citet{Helgason.new}.

\begin{definition}[{$d$}-plane]
A $d$-plane $\xxi \subset \RR^m$ is a $d$-dimensional affine subspace in $\RR^m$. Here, \emph{affine} emphasizes that it does \emph{not} always pass through the origin $\oo \in \RR^m$. 
Let $G_{m,d}$ denote the collection of all $d$-planes in $\RR^m$, called the \emph{affine Grassmannian manifold}.
\end{definition}

A $d$-plane is parametrized by its orthonormal directions $U = [\uu_1,\ldots,\uu_k] \in V_{m,k}$ and coordinate vector $\bb \in \RR^k$ from the origin $\oo$ as below
\[\xxi(U,\bb) \coloneq  U\bb + \ker U = \sum_{i=1}^k b_i \uu_i + \left\{ \sum_{j=1}^d c_j \vv_j \mBiggvert  c_j \in \RR \right\},\]
where $[\vv_1, \ldots, \vv_d] \in V_{m,d}$ is a $d$-frame satisfying $\vv_j \perp \uu_i$ for any $\forall i,j$.
The first term $U\bb$ is the displacement vector from the origin $\oo$, its norm $\svert U\bb\svert $ is the distance from the origin $\oo$ and $d$-plane $\xxi$, and the second term $\ker U = (\hull U)^\perp$ is the $d$-dimensional linear subspace that is parallel to $\xxi$.

Recall that for each direction $U \in V_{m,d}$, the whole space $\RR^m$ can be decomposed into a disjoint union of $d$-planes as $\RR^m = \cup_{\bb \in \RR^k} \xxi(U,\bb)$.
In this perspective, the $d$-plane transform of $f$ at $\xxi$ is defined as a \emph{marginalization} of $f$ in $\xxi$.

\begin{definition}[{$d$}-plane Transform]
    For any integrable function $f \in L^1(\RR^m)$ and $d$-plane $\xxi=(U,\bb) \in V_{m,k} \times \RR^k$, put
\[
        P_d[f](\xxi)
        \coloneq  \int_{\xxi} f(\xx) \dd \mathsf{L}^d(\xx) = \int_{\ker U} f(U\bb + \yy) \dd \yy,\]
where $\mathsf{L}^d$ is the $d$-dimensional Hausdorff measure on $\xxi$.
\end{definition}
Particularly, the strict Radon transform corresponds to $d=m-1$, and the $X$-ray transform corresponds to $d=1$.

The $d$-plane transform has the following Fourier expression.

\begin{lemma}[Fourier Slice Theorem for {$d$}-plane Transform]
For any $f \in L^1(\RR^m)$,
\[
    P_d[f]^\sharp(U,\oomega) = \widehat{f}(U\oomega), \quad (U,\oomega) \in V_{m,k}\times\RR^k,\]
where $\widehat{\cdot}$ and $\cdot^\sharp$ denote the Fourier transforms in $\xx$ and $\bb$ respectively. In other words,
\[
    \int_{\RR^k} P_d[f](U,\bb)e^{-i\bb\cdot\oomega} \dd\bb = \int_{\RR^m}f(\xx)e^{-iU\oomega\cdot\xx}\dd\xx.
\]
\end{lemma}

Using the Fourier slice theorem, we can invert the $d$-plane transform.

\begin{lemma}[Inversion Formula for {$d$}-plane Radon Transform]\xlabel{lem:d-inversion}
For any $f \in L^1(\RR^m)$,
\[
    f(\xx) = \frac{1}{(2\pi)^m}\int_{V_{m,k}\times\RR^k} \widehat{f}(U\oomega) \svert U\oomega\svert ^{m-k} e^{iU\oomega \cdot \xx} \dd \oomega \dd U, \quad \xx \in \RR^m.
\]
\end{lemma}

\begin{proof}
By the Fourier slice theorem,
\begin{align*}
&\frac{1}{(2\pi)^m}\int_{V_{m,k}\times\RR^k} (P_d[f])^\sharp(U,\oomega) e^{i U \oomega \cdot \xx} \svert U \oomega\svert ^{m-k} \dd U \dd \oomega \notag \\
    &=\frac{1}{(2\pi)^m}\int_{V_{m,k}\times\RR^k} \widehat{f}(U\oomega) e^{i U \oomega \cdot \xx} \svert U \oomega\svert ^{m-k} \dd U \dd \oomega \notag \\
    &=\frac{1}{(2\pi)^m}\int_{\RR^m} \widehat{f}(\xxi) e^{i\xxi \cdot \xx} \dd \xxi = f(\xx).
\end{align*}
Here, we change variable $\xxi = U\oomega$ and use the matrix polar integration formula \nxreflem{mpd}.
\end{proof}

\begin{remark}[Relations to Marginalization of Probability Distributions] \xlabel{sec:radon.is.marginal}
    In short, a $d$-plane $\xxi$ is a subset in $\RR^m$, it is identified with a single variable as well, and $d$-plane $\xxi$ (as a variable) is marginalized.

    Let us consider a two-variables (or bivariate) case.
    The marginalization of a probability distribution $f(x_1,x_2)$ in $x_1$ (resp. $x_2$) refers to an integral transform of $f$ into its first (resp. second) variable defined by $f_1(x_2) = \int_{\RR} f(x_1,x_2) \dd x_1$ (rep. $f_2(x_1) = \int_{\RR} f(x_1,x_2) \dd x_2$).
    
    On the other hand, 
    the $d$-plane transform of an integrable function $f$ on $\RR^2$ (i.e.~$f \in L^1(\RR^2)$) with $d=1$ (which is reduced to the classical Radon transform) is given by 
    \[
    P_d[f](s,\uu) = \int_{\RR} f(s\uu + t\uu^\perp) \dd t, \quad (t,\uu) \in \RR \times \SS^1\]
where $\SS^1$ denotes the set of unit vectors in $\RR^2$, i.e.~$\SS^1 = \{ \uu \in \RR^2 \mid \svert \uu\svert =1 \}$, and $\uu^\perp$ denotes an orthonormal vector, or a unit vector satisfying $\uu\cdot\uu^\perp = 0$ (there always exist two $\uu^\perp$'s for each $\uu$). Each $(s,\uu) \in \RR\times\SS^1$ indicates a $d$-plane $\xxi(s,\uu) = \{ s \uu + t \uu^\perp \mid t \in \RR\}$. 
    
    In particular, by fixing an orthonormal basis $\{ \uu_1, \uu_2\} \in \RR^2$,
    and identifying bivariate function $f(x_1,x_2)$ with univariate function $f(x_1\uu_1 + x_2\uu_2)$,
    the marginalization of probability distributions is identified with the following specific cases:
    \[
    f_1(x_2) = P_d[f](x_2, \uu_1), \quad f_2(x_1) = P_d[f](x_1, \uu_2).
\]
\end{remark}

\subsection{Network Design}

We define the $d$-plane (or $k$-affine) layer. 
Here, $k$ is the co-dimension of $d$, satisfying $d+k=m$.
In addition to the full-column-matrices cases $(A,\bb) \in M_{m,k}\times\RR^k$, we consider the degenerated cases $(A=aU,\bb)\in GV_{m,k}\times\RR^k$ and $(A=U,\bb)\in V_{m,k}\times\RR^k$, which correspond to several previous studies.

\begin{definition}
Let $\sigma:\RR^k\to\CC$ be a measurable function. Let $M$ denote either $M_{m,k}, GV_{m,k}$ or $V_{m,k}$.
For any function $\gamma:M\times\RR^k \to \CC$, the continuous neural network with $d$-\emph{plane (or $k$-affine) layer} is given by
\[S[\gamma](\xx) \coloneq  \int_{M \times \RR^k} \gamma(A,\bb) \sigma( A^\top \xx - \bb ) \dd A \dd \bb, \quad \xx \in \RR^m.
\]
\end{definition}
Since the null space $\ker A^\top \coloneq  \{ \xx \in \RR^m \mid A^\top \xx = 0\}$ is $d$-dimensional, 
each $d$-plane neuron $\sigma(A^\top\xx-\bb)$ has $d$-dimensions of constant directions. Therefore, $d$-plane networks are able to capture $d$-dimensional singularities in a target function $f$.

\subsection{Ridgelet Transforms and Reconstruction Formulas}

We present three variants of solutions for $d$-plane networks. We note that typical pooling layers $\sigma$ such as average pooling, max pooling, and $\ell^p$-norm are contained in the class of \emph{tempered distributions} ($\calS'$) on $\RR^k$.
The first and second theorems present dense ($A \in M_{m,k}$) and sparse ($A \in GV_{m,k}$) solutions of parameters for the same class of activation functions. Since $GV_{m,k}$ is a measure-zero subset of $M_{m,k}$, the second solution is much sparser than the first solution. The third theorem present the sparsest ($A \in V_{m,k}$) solution, by restricting the class of activation functions. It is supposed to capture characteristic solutions modern activation functions such as ReLU.

In the following, $c_{m,k} \coloneq  \int_{\SS^{k-1} \times V_{m,k-1}} \dd U \dd \oomega$.

\begin{theorem}\xlabel{thm:affine}
Let $\sigma \in \calS'(\RR^k), \rho \in \calS(\RR^k), f \in H^d(\RR^m)$. Put
\begin{align*}
    R[f;\rho](A,\bb) &\coloneq  \frac{1}{\delta(A)}\int_{\RR^m} \triangle^{d/2}[f](\xx)\overline{\rho(A^\top\xx-\bb)}\dd\xx, \quad (A,\bb) \in M_{m,k} \times \RR^k\\
    \iiprod{\sigma,\rho} &\coloneq  \frac{(2\pi)^d}{2^k c_{m,k}}\int_{\RR^k} \sigma^\sharp(\oomega)\overline{\rho^\sharp(\oomega)}\prod_{i=1}^k\svert \omega_i\svert ^{-1}\dd\oomega,
\end{align*}
where $\delta(A)$ is defined as $2^{-k}\prod_{i=1}^d d_i^{d} \prod_{i<j}(d_i^2-d_j^2)$ with $d_1 > \cdots > d_k > 0$ being the singular values of $A$.
Then, for almost every $\xx \in \RR^m$, we have
\[
    S[R[f;\rho]](\xx) = \int_{M_{m,k}\times\RR^k} R[f;\rho](A,\bb)\sigma(A^\top\xx-\bb)\dd A \dd \bb = \iiprod{\sigma,\rho}f(\xx).
\]
\end{theorem}

\begin{theorem}\xlabel{thm:similitude}
Let $s$ be a real number; let $\sigma \in \calS'(\RR^k), \rho \in \calS(\RR^k), f \in H^s(\RR^m)$. Put
\begin{align*}
    R_s[f;\rho](aU,\bb) &\coloneq  a^{m-s-1} \int_{\RR^m}\triangle^{s/2}[f](\xx)\overline{\rho(A^\top\xx-\bb)}\dd\xx, \quad (aU,\bb) \in GV_{m,k}\times\RR^k\\
    \iiprod{\sigma,\rho}_s &\coloneq  \frac{(2\pi)^d}{2^k c_{m,k}}\int_{\RR^k} \sigma^\sharp(\oomega)\overline{\rho^\sharp(\oomega)}\svert \oomega\svert ^{-(d-s+1)}\dd\oomega,
\end{align*}
Then, for almost every $\xx \in \RR^m$, we have
\[
    S[R_s[f;\rho]](\xx) = \int_{GV_{m,k}\times\RR^k} R_s[f;\rho](aU,\bb)\sigma(aU^\top\xx-\bb) \dd a \dd U \dd \bb = \iiprod{\sigma,\rho}_s f(\xx).
\]
\end{theorem}

\begin{theorem}\xlabel{thm:stiefel}
For any real number $t$, suppose that $\sigma \in \calS'(\RR^k)$ satisfy $\sigma^\sharp(\oomega)=\svert \oomega\svert ^t$ (i.e.,~$\sigma$ is the Green function of $\triangle_{\bb}^{-t/2}$). Let $f \in H^d(\RR^m)$. Put
\[R[f](U,\bb) \coloneq  \triangle_{\bb}^{(d-t)/2} P_d[f](U,\bb) = P_d[ \triangle^{(d-t)/2} f ](U,\bb), \quad (U,\bb) \in V_{m,k} \times \RR^k\]
where $\triangle_{\bb}$ denotes the fractional Laplacian in $\bb \in \RR^k$, and $P_d$ is the $d$-plane transform. 
Then, for almost every $\xx \in \RR^m$, we have
\[
    S[R[f]](\xx) = \int_{V_{m,k}\times\RR^k} R[f](U,\bb)\sigma(U^\top\xx-\bb)\dd U \dd \bb = \frac{1}{(2\pi)^d c_{m,k}}f(\xx).
\]
\end{theorem}

As consequences, these reconstruction formulas are understood as constructive universality theorems for $d$-plane networks.
We note
(1) that, as far as we have noticed, the first result was not known,
(2) that the second result extends the ``$d$-plane ridgelet
transform'' by~\citet{Donoho.ridgelet} and~\citet{Rubin.ridgelet} (see
\nxrefsec{literature.dplain}),
and (3) that the third result extends the Radon formulas (\nxrefthm{ito.radon})
by~\citet{Carroll.Dickinson} and~\citet{Ito.Radon} as the special case
$k=1$ and $t=-1$, and recent results on ReLU-nets such as
in~\citet{Savarese2019,Ongie2020} and \citet{Parhi2021} as the special case $k=1$
and $t=-2$.

The proof is performed by systematically following the three steps as below.

\begin{proof}
We present the first case. See \nxrefapp{proof.dplane} for full proofs.

\textit{\step{1}}.
Turn to the Fourier expression:
\[
        S[\gamma](\xx)
        =\frac{1}{(2\pi)^k} \int \gamma^\sharp(A,\oomega)\sigma^\sharp(\oomega)e^{i(A\oomega)\cdot\xx} \frac{\dd A \dd \oomega}{\delta(A)}
\]

\textit{\step{2}}.
Use singular value decomposition (SVD)
\[
    A = UDV^\top, \quad (U,D,V) \in V_{m,k} \times \RR_+^k\times O(k),\]
with $\dd A/\delta(A) = \dd U \dd D \dd V$ to have
\[
=\frac{1}{(2\pi)^k} \int \gamma^\sharp(A,\oomega)\sigma^\sharp(\oomega)e^{i(UDV^\top\oomega)\cdot\xx} \dd U \dd D \dd V \dd \oomega.
\]
Change variables $\oomega'=V^\top\oomega$ ($V$ fixed) and $\yy=D \oomega'$ ($\oomega'$ fixed)
\[
=\frac{1}{(2\pi)^k} \int \gamma^\sharp(A,V\oomega')\sigma^\sharp(V\oomega')e^{i(U\yy)\cdot\xx} \prod_{i=1}^k \svert \omega_i'\svert ^{-1} \dd U \dd \yy \dd V \dd \oomega'.
\]

\textit{\step{3}}.
Put a separation-of-variables form (note: $A\oomega = A V \oomega' = U\yy$)
\[
        \gamma^\sharp_{f,\rho}(A,V\oomega') = \widehat{f}(U\yy)\svert U\yy\svert ^{m-k}\overline{\rho(V\oomega')}
\]
Then, $\gamma_{f,\rho}$ turns out to be a particular solution because
\begin{align*}
        S[\gamma_{f,\rho}](\xx)
        &=\frac{c_{m,k}}{(2\pi)^k} 
        \left(\int_{O(k)\times\RR^k} \sigma^\sharp(V\oomega') \overline{\rho^\sharp(V\oomega')} \prod_{i=1}^k \svert \omega_i'\svert ^{-1} \dd V \dd \oomega'\right) \\
        &\quad \times\left(
        \int_{V_{m,k}\times\RR^k} \widehat{f}(U\yy)\svert U\yy\svert ^{m-k} e^{i(U\yy)\cdot\xx} \dd U \dd \yy\right)\\
        &= \iiprod{\sigma,\rho} f(\xx).
\end{align*}
Finally, the matrix ridgelet transform can be calculated as below
\begin{align*}
        \gamma_{f,\rho}(A,\bb)
        &= \frac{1}{(2\pi)^k}\int_{\RR^k} \svert A\oomega\svert ^{m-k}\widehat{f}(A\oomega)\overline{\rho(\oomega)}e^{i\oomega\cdot\bb}\dd\oomega \\
        &= \frac{1}{(2\pi)^k} \int_{\RR^k} \left[\int_{\RR^m} \triangle^{(m-k)/2} [f](\xx) e^{-iA\oomega\cdot\xx} \dd \xx \right]\overline{\rho^\sharp(\oomega)}e^{i\oomega \cdot \bb} \dd\oomega \\
        &= \int_{\RR^m} \triangle^{(m-k)/2} [f](\xx) \left[ \frac{1}{(2\pi)^k} \int_{\RR^k} \rho^\sharp(\oomega)e^{i\oomega \cdot (A^\top\xx-\bb)} \dd\oomega\right]^{*} \dd\xx \\
        &= \int_{\RR^m} \triangle^{\frac{m-k}{2}}[f](\xx) \overline{\rho(A^\top\xx-\bb)}\dd\xx \\
        &\eqcolon  R[f;\rho](A,\bb) 
\end{align*}
\end{proof}

\subsection{Literature in $d$-plane ridgelet transform}
 \xlabel{sec:literature.dplain}
In the past, two versions of the $d$-plane ridgelet transform have
been proposed. One is a tight frame (i.e.,~a discrete transform)
by~\citet{Donoho.ridgelet}, and the other is a continuous transform
by~\citet{Rubin.ridgelet}.
The $d$-plane ridgelet by Donoho can be regarded as the discrete version of the $d$-plane ridgelet transform by Rubin. 

\begin{theorem}[{Continuous $d$-plane Ridgelet Transform by~\citet{Rubin.ridgelet}}]\xlabel{thm:dplane.rubin}
\begin{align*}
&U_a[f](\xxi) \coloneq  \int_{\RR^m} f(\xx) u_a\left(  \svert \xx-\xxi\svert \right) \dd\xx, \quad \xxi \in G_{m,d}\\
    &V^{*}_a[\phi](\xx) \coloneq  \int_{G_{m,d}} \phi(\xxi) v_a\left(  \svert \xx-\xxi\svert \right) \dd\xxi, \quad \xx\in\RR^m\\
    &\int_0^\infty V_a^{*} [U_a [f]] \frac{\dd a}{a^{1+d}} = c f,
\end{align*}
where $\svert \xx-\xxi\svert $ denotes the Euclidean distance between point $\xx$ and $d$-plane $\xxi$, $u_a(\cdot) = u(\cdot/a)/a^d$ and $v_a(\cdot) = v(\cdot/a)/a^d$. 
\end{theorem}
Recall that an affine $d$-plane $\xxi \in G_{m,k}$ is parametrized by an orthonormal $k$-frame $U \in V_{m,k}$ and a coordinate vector $\bb \in \RR^k$ as $\xxi(U,\bb) \coloneq  \{ \xx\in\RR^m \mid U^\top\xx=\bb \} = U\bb + \ker U$. 
Because $\svert \xx-\xxi(U,\bb)\svert  = \svert U^\top\xx-\bb\svert $ for any point $\xx\in\RR^m$, the quantity $U^\top \xx - \bb$ is understood as the Euclidean \emph{vector-distance} between point $\xx$ and $d$-plane $\xxi$. Therefore, the $d$-plane ridgelet transform by Rubin is understood as a special case of $\sigma(aU^\top\xx-\bb)$ as in \nxrefthm{similitude} where both $\sigma$ and $\rho$ are radial functions. We remark that a more redundant parametrization $\sigma(A^\top\xx-\bb)$ as in \nxrefthm{affine} is natural for the purpose of neural network study, simply because neural network parameters are not strictly restricted to $\sigma(aU^\top\xx-\bb)$ during the training.

 \section{Literature Overview}

\subsection{Ridgelet Transform in the 1990s}

One of the major problems in neural network study in the 1990s was to
investigate the expressive power of (fully-connected) shallow neural networks,
and the original ridgelet transform was discovered in this context
independently by~\citet{Murata1996}, and~\citet{Candes.PhD}.
Later in the 2010s, the classes of $f$ and $\sigma$ have been
extended to the distributions by~\citet{Kostadinova2014}
and~\citet{Sonoda2015acha} to include the modern activation functions such as
ReLU.

The idea of using integral transforms for function approximation is fundamental
and has a long history in approximation theory~\citep[see, e.g.][]{DeVore1993}.
In the literature of neural network study, the integral representation
by~\citet{Barron1993} is one of the representative works, where the so-called
Barron class and Maurey--Jones--Barron (MJB) approximation error upperbound have
been established, which play an important role both in the approximation and
estimation theories of neural networks. We refer to~\citet{kainen.survey} for
more details on the MJB theory.

One obvious strength of the ridgelet transform is the closed-form expression.
Before the ridgelet transform, two pioneering results were proposed. One is the
Fourier formula by~\citet{Irie1988} and~\citet{Funahashi1989}:

\begin{theorem}
For any $\sigma\in L^1(\RR)$ and $f\in L^2(\RR^m)$,
\[
    f(\xx) = \frac{1}{(2\pi)^m \sigma^\sharp(1)}\int_{\RR^m\times\RR}\widehat{f}(\aa)\sigma(\aa\cdot\xx-b)e^{ib}\dd\aa\dd b.
\]
\end{theorem}

The other is the Radon formula by~\citet{Carroll.Dickinson}
and~\citet{Ito.Radon}: 

\begin{theorem}\xlabel{thm:ito.radon}
For $\sigma(b) \coloneq  b_+^0$ (step function) and any $f\in\calS(\RR^m)$,
\[
    f(\xx) = \frac{1}{2(2\pi)^{m-1}}\int_{\SS^{m-1}\times\RR} \partial_t (-\triangle_t)^{(m-1)/2} P[f](\uu,t) \sigma(\uu\cdot\xx-t)\dd\uu\dd t,\]
where $P$ denotes the Radon transform.
\end{theorem}
 Both results clearly show the strong relationship between neural networks and the Fourier and Radon transforms. We note that our result \nxrefthm{stiefel} includes the Radon formula (\nxrefthm{ito.radon}) as the special case $k=1$ and $t=-1$.

\subsection{Ridgelet Transform in the 2000s}
 \xlabel{sec:literature2000}
In the context of sparse signal processing, the emergence of the ridgelet transform has motivated another direction of research: Exploring a high-dimensional counterpart of the $1$-dimensional wavelet transform. Indeed, the wavelet transform for $m$-dimensional  signals such as images and videos \emph{does} exist, and it is given as below
\begin{align*}
&W[f;\psi](a,\bb) \coloneq  \int_{\RR^m\times\RR^m} f(\xx) \overline{\psi_a(\xx-\bb)}\dd\xx, \quad (a,\bb) \in \RR_+\times\RR^m\\
    &f(\xx) = \int_{\RR^m\times\RR_+} W[f;\psi] \psi_a( \xx-\bb ) \frac{\dd\bb\dd a}{a}, \quad \xx \in \RR^m
\end{align*}
for $f \in L^2(\RR^m)$, where $\psi_a(\bb) \coloneq  \psi(\bb/a)/a^m$ and $\psi \in \calS(\RR^m)$ is a wavelet function.
However, it is considered to be \emph{unsatisfied} in its \emph{localization} ability, because it is essentially a tensor product of $1$-dimensional wavelet transforms.

More precisely, while the $1$-dimensional wavelet transform is good at
localizing the point singularities such as jumps and kinks in the
$1$-dimensional signals such as audio recordings, the
$2$-dimensional wavelet transform is \emph{not} good at localizing the
line singularities in $2$-dimensional signals such as pictures except
when the singularity is straight and parallel to either $x$- or
$y$-axes. Here, the \emph{singularity of dimension $d$} is
the term by~\citet{Donoho.ridgelet}. For example, 
\[
f(\xx)\coloneq ( x_1^2+ \cdots + x_k^2 )^{-\alpha/2} \exp(-\svert\xx\svert^2), \quad  0 < \alpha < k/2\]
is a singular square-integrable function $f$ on $\RR^m$ that attains $\infty$ along the hyperplane $x_1 = \cdots = x_k = 0$. 

On the other hand, the ridgelet transform is good at localizing the $(m-1)$-dimensional singularities in any direction because the feature map $\xx\mapsto\sigma(\aa\cdot\xx-b)$ is a constant function along the $(m-1)$-dimensional hyperplane normal to $\aa$. 
Similarly, the $d$-plane (or $k$-affine) ridgelet transform presented in this study is good at localizing the $d$-dimensional singularities because the feature map $\xx\mapsto\sigma(A^\top\xx-\bb)$ is a constant function along the $d$-dimensional subspace $\ker A^\top=\{ \xx \in \RR^m \mid A^\top \xx = 0\}$. 

In search for better localization properties, a variety of ``X-lets'' have been
developed such as curvelet, beamlet, contourlet, and sheerlet under the slogan
of geometric multiscale analysis (GMA)~\citep[see e.g.~][]{Donoho2002,Starck2010}.
Since ridgelet analysis had already been recognized as \emph{wavelet analysis
in the Radon domain}, a variety of generalizations of wavelet transforms and
Radon transforms were investigated. In a modern sense, the philosophy of
general Radon transforms is to map a function $f$ on a space
$X=G/K$ of points $x$ to a function $P[f]$ on another space
$\Xi=G/H$ of shapes $\xi$~\citep[see e.g.][]{Helgason.new}. In the
context of singularity localization, the shape $\xi$ such as
$d$-plane determines the shape of singularities, namely, a collection
of constant directions in $X$, and thus the general Radon domain
$\Xi$ is understood as the parameter space of the singularities. In this
perspective, we can restate the functionality of the ridgelet transform as
\emph{wavelet localization in the space $\Xi$ of singularities in
$X$.}

\subsection{Ridgelet Transform in the 2020s}

In the context of deep learning study, the idea of ridgelet transforms have
regained the spotlight for the \emph{representer theorem} that characterizes
(either deep or shallow) infinitely-wide ReLU networks that minimizes a
``representational cost''~\citep{Savarese2019,Ongie2020,Parhi2021,Unser2019}.
Here, the representational cost for function $f$ is defined as the
infimum of the total variation (TV) norm of the parameter distribution: 
\[
C[f] \coloneq  \inf_{\gamma \in \calG} \|  \gamma \| _{\rm TV}, \quad \mbox{s.t.} \quad S[\gamma]=f,\]
where $\calG$ is the collection of all signed measures. The TV-norm is
a fundamental quantity for
the MJB bounds~\citep[see e.g.][]{kainen.survey}.

According to~\citet{Sonoda2021ghost}, when the class $\calG$ of parameter
distributions is restricted to $L^2(\RR^m\times\RR)$, then any $\gamma$ satisfying
$S[\gamma]=f$ is uniquely written as a series of ridgelet transforms:
$\gamma = R[f;\sigma_*] + \sum_{i=1}^\infty R[f_i;\rho_i]$ where $\sigma_*$ is a certain unique function satisfying
$\iiprod{\sigma,\rho_0}=1$, yielding $S[R[f;\sigma_*]]=f$; $\{\rho_i\}_{i \in \NN}$ is an orthonormal system
satisfying $\iiprod{\sigma,\rho_i}=0$, yielding $S[R[f_i;\rho_i]]=0$; and $\{f_i\}_{i\in\NN}$ is
$L^2$-functions that is uniquely determined for each $\gamma$. We
note that $\sigma_*$ and $\rho_i$ are independent of $\gamma$. Hence,
the cost is rewritten as a constraint-free expression:
\[
    C[f] = \inf_{\{f_i\}} \Big\Vert  R[f;\sigma_*] + \sum_{i=1}^\infty R[f_i;\rho_i] \Big\Vert_{L^1}.
\]

As a result, we can conjecture that the minimizer of $C[f]$ is given by
ridgelet transform(s). In fact,~\citet{Ongie2020} have shown that under some
assumptions, the minimizer is given by a derivative of the Radon transform:
$\triangle^{(m+1)/2} P[f]$, which is exactly the special case of the ridgelet transform in
\nxrefthm{stiefel} when $k=1$ and $t=-1$. 

Update: At the same time as the initial submission, \citet{parhi2023functionspace} have obtained a representer theorem for \emph{multivariate} activation functions under more careful considerations on the regularization and function spaces based on an extended theory of the $d$-plane transforms for \emph{distributions} \citep{parhi2023distributional}. Their result suggests our conjecture was essentially true (modulo finite-order polynomials).

 \section{Discussion}

In the main text, we have seen a variety of examples, but what is essential behind the Fourier expression, changing variables and assuming the separation-of-variables form? In a nutshell, it is \emph{coefficient comparison} for solving equations. Namely, after appropriately changing variables, the network $S[\gamma]$ is rewritten in the Fourier basis, which is thus the coordinate transform from the basis $\{\sigma(ax-b)\}_{a,b}$ to the Fourier basis $\{\exp(i\xi x)\}_\xi$. Since we (are supposed to) know the Fourier coefficient $\widehat{f}(\xi)$, we can obtain the unknown function $\gamma$ by comparing the coefficients in the Fourier domain. From this perspective, we can now understand that the Fourier basis is just a one choice of frames, and the solution steps are summarized as below:

Let $\phi:V \times X \to \RR$ and $\psi:\Xi\times X\to \RR$ be two feature maps on $X$ parametrized by $V$ and $\Xi$ respectively, and consider their associated integral representations:
\[
    S[\gamma](x) \coloneq  \int_V \gamma(v)\phi(v,x)\dd v, \quad T[g](x) \coloneq  \int_\Xi g(\xi)\psi(\xi,x)\dd\xi.
\]
Here, $S$ and $T$ correspond to the continuous neural network and the inverse Fourier transform respectively. Given a function $f:X\to\RR$, suppose that $g$ of $T[g]=f$ is known as, say $g=\widehat{f}$, but $\gamma$ of $S[\gamma]=f$ is unknown. Then, find a coordinate transform $H$ satisfying $H[ \phi ](\xi,x) = \psi(\xi,x)$ so that
\[
    S[\gamma](x) = \int_\Xi H'[\gamma](\xi) \psi(\xi,x)\dd\xi,\]
where $H'$ is a dual transform of the coefficients $\gamma$ associated with the coordinate transform. Then, we can find $\gamma$ by comparing the coefficients:
\[
    H'[\gamma] = \widehat{f}.
\]
In other words, the mapping $H$ that matches the neuron $\sigma(\aa\cdot\xx-b)$ and Fourier basis $\exp(i\xxi\cdot\xx)$ corresponds to the Fourier expression and change of variables in the main text, yielding $H'[\gamma](\xxi) = \gamma^\sharp(\xxi/\omega,\omega)$.

\section{Conclusion}

The ultimate goal of this study is to understand neural network parameters.
While the ridgelet transform is a strong analysis tool, one of the major short-comings is that the closed-form expression has been known only for small class of neural networks. 
In this paper, we propose the Fourier slice method, and have shown that various neural networks and their corresponding ridgelet transforms, listed in \nxreftab{types}, can be systematically obtained by following the three steps of the Fourier slice method.

\begin{table*}[t]
    \small
    \centering
    \begin{tabular}{llllll}
        \toprule
        \multicolumn{2}{l}{layer type} & input $x$& parameter $(a,b) $& single neuron\\
        \midrule
        \multicolumn{2}{l}{\refsec{intro}-\ref{sec:method}. fully-connected (FC) layer} & $\RR^m$ & $\RR^m \times \RR$ & $\sigma(\aa\cdot\xx-b)$ \\
        \multicolumn{2}{l}{\refsec{case.finite}. FC layer on finite fields} & $\FF_p^m$ & $\FF_p^m \times \FF_p$ & $\sigma(\aa\cdot\xx-b)$ \\
        \multicolumn{2}{l}{\refsec{case.gconv}. group convolution layer} & $\calH_m$ & $\calH_m \times \RR$ & $\sigma( (a \star x)(g) - b)$\\
        \multicolumn{2}{l}{\refsec{case.mfd}. FC layer on manifolds} & $G/K$ & $\lieA^* \times \bdX \times \RR$ & $\sigma(a\iprod{x,u}-b)$\\
        \multicolumn{2}{l}{\refsec{case.affine}. pooling ($d$-plane ridgelet)} & $\RR^m$ & $\RR^{m \times k} \times \RR^k$ & $\sigma(A^\top\xx-\bb)$\\
        \bottomrule
    \end{tabular}
    \caption{List of layer types $\sigma(ax-b)$ covered in this study. See corresponding sections for the definitions of symbols such as $\FF_p, \calH_m, G/K, \partial X$ and $\lieA^*$.}
    \label{tab:types}
\end{table*}
Needless to say, it is more efficient to analyze networks uniformly in terms of ridgelet analysis than to analyze individual networks manually one by one. 
As demonstrated in this paper, the coverage of ridgelet analysis is gradually expanding.
With the strength of a closed-form expression of the pseudo-inverse operator,
the ridgelet transform has several applications. For example, we can/may
\begin{enumerate}
    \item present a constructive proof of the universal approximation theorem,
    \item estimate approximation error bounds by discretizing the reconstruction formula using numerical integration schemes (e.g.~MJB theory),
    \item describe the distribution of parameters obtained by gradient descent
     learning~\citep{Sonoda2021aistats},
    \item obtain the general solution to the learning equation
     $S[\gamma]=f$~\citep{Sonoda2021ghost}, and
    \item construct a representer theorem~\citep{Unser2019}.
\end{enumerate}

The Fourier expression further allows us to view neural networks from the
perspective of harmonic analysis and integral geometry. By recasting neural
networks in these contexts, we will be able to discover further varieties of
novel networks. On the other hand, after the initial submission of this
manuscript, the authors have also developed an alternative method of discovery
that uses group invariant functions instead of the Fourier
expression~\citep{sonoda2023deepridge,sonoda2023joint}.
The characterization of the networks obtained by the Fourier slice method would be an interesting direction of this study.

\section*{Acknowledgment}
This work was supported by
JSPS KAKENHI 18K18113,
JST CREST JPMJCR2015 and JPMJCR1913,
JST PRESTO JPMJPR2125,
and JST ACT-X JPMJAX2004.

\appendix
\section{Poincar\'e Disk}\xlabel{sec:poincare}

Following~\citet{Helgason.GGA}[Introduction, \textsection~4]
and~\citet{Helgason.GASS}[Ch.II, \textsection~1], we describe the group theoretic aspect
of the Poincar\'e disk.
Let $D \coloneq  \{ z \in \CC \mid \svert z\svert  < 1 \}$ be the unit open disk in $\CC$ equipped with the Riemannian metric $g_z(u,v) = \eprod{u,v}/(1-\svert z\svert ^2)^2$ for any tangent vectors $u,v \in T_z D$ at $z \in D$, where $\eprod{\cdot,\cdot}$ denotes the Euclidean inner product in $\RR^2$.
Let $\bdD \coloneq  \{ u \in \CC \mid \svert u\svert =1\}$ be the boundary of $D$ equipped with the uniform probability measure $\dd u$.
Namely, $D$ is the \emph{Poincar\'e disk model of hyperbolic plane $\HH^2$}.
On this model, the Poincar\'e metric between two points $z,w \in D$ is given by $d(z,w)=\tanh^{-1} \svert  (z-w)/(1-zw^{*}) \svert $, and the volume element is given by $\dd z = (1-(x^2+y^2))^{-2}\dd x \dd y$.

Consider now the group
\[
G \coloneq  SU(1,1) \coloneq  \left\{ 
\begin{pmatrix}
 \alpha & \beta \\ \beta^{*} & \alpha^{*} 
\end{pmatrix}
 \mBiggvert  (\alpha,\beta)\in\CC^2, \svert \alpha\svert ^2-\svert \beta\svert ^2=1 \right\},\]
which acts on $D$ (and $\bdD$) by
\[
    g \cdot z \coloneq  \frac{\alpha z+\beta}{\beta^{*}z+\alpha^{*}}, \quad z \in D \cup \bdD.
\]
The $G$-action is transitive, conformal, and maps circles, lines, and the boundary into circles, lines, and the boundary.
In addition, consider the subgroups
\begin{align*}
    K &\coloneq  SO(2) = \left\{ k_\phi \coloneq  
\begin{pmatrix}
 e^{i \phi} & 0 \\ 0 & e^{-i\phi} 
\end{pmatrix}
 \mBiggvert  \phi \in [0,2\pi) \right\}, \\
    A  &\coloneq  \left\{ a_t \coloneq  
\begin{pmatrix}
 \cosh t & \sinh t \\ \sinh t & \cosh t 
\end{pmatrix}
 \mBiggvert  t \in \RR \right\}, \\
    N  &\coloneq  \left\{ n_s \coloneq  
\begin{pmatrix}
 1+is & -is \\ is & 1-is 
\end{pmatrix}
 \mBiggvert  s \in \RR \right\}, \\
    M &\coloneq  C_K(A) = \left\{ k_0 = 
\begin{pmatrix}
 1 & 0 \\ 0 & 1 
\end{pmatrix}
, k_\pi = 
\begin{pmatrix}
 -1 & 0 \\ 0 & -1 
\end{pmatrix}
 \right\}
\end{align*}
The subgroup $K \coloneq  SO(2)$ fixes the origin $o \in D$. So we have the identifications
\[
    D = G/K = SU(1,1)/SO(2), \quad \mbox{and} \quad \bdD = K/M = \SS^{1}.
\]

On this model, the following are known 
(1) that $m=\dim \lieA=1$, $\svert W\svert =1$, $\varrho=1$, and $\svert \cc(\lambda)\svert ^{-2} = \frac{\pi \lambda}{2}\tanh (\frac{\pi \lambda}{2})$ for $\lambda \in \lieA^{*} = \RR$,
(2) that the geodesics are the circular arcs perpendicular to the boundary $\bdD$, and (3) that the horocycles are the circles tangent to the boundary $\bdD$. Hence, let $\xi(x,u)$ denote the horocycle $\xi$ through $x \in D$ and tangent to the boundary at $u \in \bdD$; and let $\iprod{x,u}$ denote the signed distance from the origin $o \in D$ to the horocycle $\xi(x,u)$.

In order to compute the distance $\iprod{z,u}$, we use the following fact: 
The distance from the origin $o$ to a point $z=r e^{iu}$ is $d(o,z) = \tanh^{-1}\svert  (0-z)/(1-0z^{*}) \svert  = \frac{1}{2}\log\frac{1+r}{1-r}$. Hence, let $c \in D$ be the center of the horocycle $\xi(z,u)$, and let $w \in D$ be the closest point on the horocycle $\xi(z,u)$ to the origin. By definition, $\iprod{z,u} = d(o,w)$. But we can find the $w$ via the cosine rule:
\[
    \cos zou = \frac{\svert u\svert ^2+\svert z\svert ^2-\svert z-u\svert ^2}{2\svert u\svert\,\svert z\svert } = \cos zoc = \frac{\svert z\svert ^2 + {\svert \frac{1}{2}(1+\svert w\svert )\svert }^2 - {\svert \frac{1}{2}(1-\svert w\svert )\svert} ^2}{2\svert z\svert\,\svert \frac{1}{2}(1+\svert w\svert )\svert },\]
which yields the tractable formula:
\[
\iprod{z,u} = \frac{1}{2} \log \frac{1+\svert w\svert }{1-\svert w\svert } = \frac{1}{2} \log \frac{1-\svert z\svert ^2}{\svert z-u\svert ^2}, \quad (z,u) \in D \times \bdD.
\]

\section{{{SPD}} Manifold}\xlabel{sec:spdmfd}

Following~\citet[][Chapter~1]{Terras2016}, we introduce the SPD manifold.
On the space $\PP_m$ of $m \times m$ symmetric positive definite (SPD) matrices, the Riemannian metric is given by
\[
    \mathfrak{g}_{x} \coloneq  \tr\left( (x^{-1} \dd x)^2 \right), \quad x \in \PP_m\]
where $x$ and $\dd x$ denote the matrices of entries $x_{ij}$ and $\dd x_{ij}$.

Put $G=GL(m,\RR)$, then the Iwasawa decomposition $G=KAN$ is given by $K=O(m), A=D_+(m), N=T_1(m)$; and the centralizer $M = C_K(A)$ is given by $M=D_{\pm 1}$ (diagonal matrices with entries $\pm 1$). The quotient space $G/K$ is identified with the SPD manifold $\PP_m$ via a diffeomorphism onto, $gK \mapsto g g^\top$ for any $g \in G$; and $K/M$ is identified with the boundary $\bdP_m$, another manifold of all \emph{singular positive semidefinite} matrices. The action of $G$ on $\PP_m$ is given by $g[x] \coloneq  g x g^\top$ for any $g \in G$ and $x \in \PP_m$. In particular, the metric $\mathfrak{g}$ is $G$-invariant. According to the \emph{spectral decomposition}, for any $x \in \PP_m$, there uniquely exist $k \in K$ and $a \in A$ such that $x = k[a]$; and according to the \emph{Cholesky (or Iwasawa) decomposition}, there exist $n \in N$ and $a \in A$ such that $x = n[a]$. 

When $x = k[\exp(H)] = \exp(k[H])$ for some $H \in \lieA = D(m)$ and $k \in K$, then the geodesic segment $y$ from the origin $o=I$ (the identity matrix) to $x$ is given by 
\[
y(t) = \exp(t k[H]), \quad t \in [0,1]\]
satisfying $y(0) = o$ and $y(1) = x$;
and the Riemannian length of $y$ (i.e.,~the Riemannian distance from $o$ to $x$) is given by $d(o,x) = \svert H\svert _E$. So, $H \in \lieA$ is the \emph{vector-valued distance} from $o$ to $x = k[\exp(H)]$.

The $G$-invariant measures are given by $\dd g = \svert \det g\svert ^{-m} \bigwedge_{i,j} \dd g_{ij}$ on $G$, $\dd k$ to be the uniform probability measure on $K$, $\dd a = \bigwedge_i \dd a_i/a_i$ on $A$, $\dd n = \bigwedge_{1 < i < j \le m} \dd n_{ij}$ on $N$,
\begin{align*}
\dd \mu(x) 
&= \svert \det x\svert ^{-\frac{m+1}{2}} \bigwedge_{1 \le i \le j \le m} \dd x_{ij} \quad \mbox{on} \quad \PP_m, \notag \\
&= c_m \prod_{j=1}^m a_j^{-\frac{m-1}{2}} \prod_{1 \le i < j \le m} \svert a_i - a_j\svert  \dd a \dd k,
\end{align*}
where the second expression is for the polar coordinates $x = k[a]$ with $(k,a) \in K \times A$ and $c_m \coloneq  \pi^{(m^2+m)/4} \prod_{j=1}^m j^{-1} \Gamma^{-1}(j/2)$,
and $\dd u$ to be the uniform probability measure on $\bdP_m \coloneq  K/M$.

The vector-valued composite distance from the origin $o$ to a horosphere $\xi(x,u)$ is calculated as 
\[
    \iprod{x = g[o], u=kM} = \frac{1}{2} \log \lambda(k^\top[x]),\]
where $\lambda(y)$ denotes the diagonal vector $\lambda$ in the \emph{Cholesky decomposition} $y = \nu[\lambda] = \nu \lambda \nu^\top$ of $y$ for some $(\nu,\lambda) \in NA$.

\begin{proof}
Since $\iprod{x,kM} \coloneq  -H(g^{-1}k) = \iprod{k^\top [x],eM}$,
it suffices to consider the case $(x,u)=(g[o],eM)$. Namely, we solve $g^{-1} = kan$ for unknowns $(k,a,n) \in KAN$. (To be precise, we only need $a$ because $\iprod{x,eM} = - \log a$.) Put the Cholesky decomposition $x = \nu[\lambda] = \nu \lambda \nu^\top$ for some $(\nu,\lambda) \in NA$. Then, $a = \lambda^{-1/2}$ because $x^{-1} = (\nu^{-1})^\top \lambda^{-1} \nu^{-1}$, while $x^{-1} = (g g^\top)^{-1} = n^\top a^2 n$.
\end{proof}

The Helgason--Fourier transform and its inversion formula are given by
\begin{align*}
&\widehat{f}(\ss,u) = \int_{\PP_m} f(x) \overline{e^{\ss\cdot\iprod{x,u}}} \dd \mu(x),\\
    &f(x) = \omega_m \int_{\Re \ss = \rrho} \int_{\bdP_m} \widehat{f}(\ss,u) e^{\ss\cdot\iprod{x,u}} \dd u \frac{\dd \ss}{\svert \cc(\ss)\svert ^2},
\end{align*}
for any $(\ss,u) \in \lieA^{*}_{\CC} \times O(m)$ (where $\lieA^{*}_{\CC} = \CC^m$) and $x \in \PP_m$. Here, $\omega_m \coloneq  \prod_{j=1}^m \frac{\Gamma(j/2)}{j (2\pi i) \pi^{j/2}}$, $\rrho = (-\frac{1}{2},\ldots,-\frac{1}{2},\frac{m-1}{4}) \in \CC^m$, and 
\[
    \cc(\ss) = \prod_{1 \le i \le j < m} \frac{B(\frac{1}{2}, s_i + \cdots + s_j + \frac{j-i+1}{2})}{B(\frac{1}{2},\frac{j-i+1}{2})},\]
where $B(x,y) \coloneq  \Gamma(x)\Gamma(y)/\Gamma(x+y)$ is the beta function.

\section{Matrix Calculus} \xlabel{sec:BPformula}
We refer to~\citet{Rubin2018} and~\citet{Diaz-Garcia2005} for matrix calculus.

Let $m,k$ be positive integers ($m \ge k$).
Let $M_{m,k} \subset \RR^{m\times k}$ be the set of all full-column-rank matrices equipped with the volume measure $\dd W = \prod_{i,j} \dd w_{ij}$.
Let $V_{m,k}$ be the Stiefel manifold of orthonormal $k$-frames in $\RR^m$ equipped with the invariant measure $\dd U$ normalized to $\int_{V_{m,k}}\dd U = 2^k \pi^{mk/2}/\Gamma_k(m/2) \eqcolon  \sigma_{m,k}$ where $\Gamma_k$ is the Siegel gamma function.
Let $\sym_k \subset \RR^{k\times k}$ be the space of real symmetric matrices equipped with the volume measure $\dd S = \prod_{i\le j} \dd s_{ij}$, which is isometric to the euclidean space $\RR^{k(k+1)/2}$.
Let $\pdm_k \subset \RR^{k\times k}$ be the cone of positive definite matrices in $\sym_k$ equipped with the volume measure $\dd P = \prod_{i\le j} \dd p_{ij}$.

\begin{lemma}[Matrix Polar Decomposition]\xlabel{lem:mpd}
For any $W \in M_{m,k}$, there uniquely exist $U \in V_{m,k}$ and $P \in \pdm_k$ such that
\[
    W=UP^{1/2}, \quad P=W^\top W;\]
and for any $f \in L^1(M_{m,k})$,
\[
\int_{M_{m,k}} f( W ) \dd W
= \frac{1}{2^k}\int_{V_{m,k}\times\pdm_k} f( U P^{1/2} ) \svert  \det P \svert ^{\frac{m-k-1}{2}} \dd P \dd U, \quad P = W^\top W. 
\]
\end{lemma}

See~\citet[Lemma~2.1]{Rubin2018} for more details. We remark that while \nxreflem{mpd} is an integration formula on Stiefel manifold, the Grassmannian manifold version is the Blaschke--Petkantschin formula.

\begin{lemma}[Matrix Polar Integration]
For any $f \in L^1(\RR^k)$,
\[
    c_{m,k} \int_{\RR^m} f(\xx)\dd\xx = \int_{V_{m,k}\times\RR^k} f(U\bb) \svert U\bb\svert ^{m-k} \dd U \dd \bb,\]
where $c_{m,k} \coloneq  \int_{\SS^{k-1} \times V_{m,k-1}} \dd U \dd \oomega$.
\end{lemma}

\begin{proof}
Recall that $U \bb = \sum_{i=1}^k b_i \uu_i \in \RR^m$ and thus $\svert U\bb\svert ^2=\bb^\top U^\top U \bb = \svert \bb\svert ^2$.
Hence, using the polar decomposition $\bb = r \oomega$ with $\dd \bb = r^{k-1}\dd r \dd \oomega$,
\begin{align*}
&\int_{V_{m,k}\times\RR^k} f(U\bb) \svert U\bb\svert ^{m-k} \dd U \dd \bb\\
&=\int_{V_{m,k}\times\SS^{k-1}\times\RR_+} f(U \oomega r) r^{m-k} r^{k-1} \dd r \dd \oomega \dd U
\end{align*}
{then letting $\uu_{\oomega} \coloneq  \sum_{i=1}^k \omega_i \uu_i$, which is a unit vector in $\hull \, U=[\uu_1, \ldots, \uu_k]$, and letting $U_{-\oomega}$ be a rearranged $k-1$-frame in $\hull \, U$ that excludes $\uu_{\oomega}$,}
\begin{align*}
&=\int_{\SS^{k-1}} \left[ \int_{V_{m,k-1}\times\SS^{k-1}\times\RR_+} f( r \uu_{\oomega} ) r^{m-1} \dd r \dd \uu_{\oomega} \dd U_{-\oomega} \right] \dd \oomega\\
&=\int_{\SS^{k-1} \times V_{m,k-1}} \dd U_{k-1} \dd \oomega \int_{\RR^m} f( \xx ) \dd \xx, \quad \xx = r \uu_{\oomega}. 
\end{align*}
\end{proof}

\begin{lemma}[SVD]\xlabel{lem:svd}
For any column-full-rank matrix $W \in M_{m,k}$, there exist $(U,D,V) \in V_{m,k}\times\RR_+^k\times O(k)$ satisfying $W = UDV^\top$ with $d_1 > \cdots > d_k > 0$; and for any $f \in L^1(M_{m,k})$,
\[
\int_{M_{m,k}}f(W)\dd W = 2^{-k} \int_{V_{m,k}\times\RR_+^k\times O(k)} f(UDV^\top) \svert \det D\svert ^{m-k}\prod_{i<j}( d_i^2 - d_j^2 ) \dd D \dd V \dd U,\]
where $\dd D = \bigwedge_{i=1}^k \dd d_i$ and $\dd U,\dd V$ denote the invariant measures.
\end{lemma}

See Lemma 1 of \citet{Diaz-Garcia2005} for the proof.

\section{Proofs}\xlabel{sec:proof.dplane}

\subsection{Solution via Singular Value Decomposition}\xlabel{sec:svd}

\paragraph*{\textit{\step{1}}.}
We begin with the following Fourier expression:
\numberwithin{equation}{section}
\setcounter{equation}{0}
\begin{align}
S[\gamma](\xx)
&= \int_{M_{m,k} \times \RR^k} \gamma(A,\bb) \sigma( A^\top \xx - \bb ) \dd A \dd \bb \notag \\
&= \frac{1}{(2\pi)^k}\int_{M_{m,k} \times \RR^k} \gamma^\sharp(A,\oomega) \sigma^\sharp(\oomega)e^{i (A\oomega) \cdot \xx} \dd A \dd \oomega. \xlabel{eq:fourier}
\end{align}
Here, we assume \nxrefeq{fourier} to be absolutely convergent for a.e. $\xx\in\RR^m$, so that we can change the order of integrations freely. But this assumption will be automatically satisfied because we eventually set $\gamma$ to be the ridgelet transform.

\paragraph*{\textit{\step{2}}.}
In the following, we aim to turn the integration $\int \cdots e^{i(A\oomega)\cdot\xx}\dd A \dd \oomega$ into the Fourier inversion $\int \cdots e^{iU\xxi\cdot\xx}\svert U\xxi\svert ^d\dd U\dd\xxi$ in the matrix polar coordinates. 
To achieve this, we use the singular value decomposition.

\begin{lemma}[Singular Value Decomposition, \nxreflem{svd}]
The matrix space $M_{m,k}$ can be decomposed into $M_{m,k} = V_{m,k}\times\RR_+^k\times O(k)$ via \emph{singular value decomposition} 
\[
A=U D V^\top, \quad (U,D,V) \in V_{m,k}\times\RR_+^k\times O(k)\]
satisfying $D=\diag[d_1,\ldots,d_k] \, (d_1 > \cdots > d_k > 0)$ (distinct singular values); and 
the Lebesgue measure $\dd A$ is calculated as
\[
    \dd A = \delta(D) \dd D \dd U \dd V,
    \quad \delta(D) \coloneq  2^{-k}\svert \det D\svert ^{d}\Delta(D^2),\]
where $\dd D = \bigwedge_{i=1}^k \dd d_i$; $\dd U$ and $\dd V$ are invariant measures on $V_{m,k}$ and $O(k)$ respectively; and $\Delta(D^2) \coloneq  \prod_{i<j}( d_i^2 - d_j^2 )$ denotes the Vandermonde polynomial (or the products of differences) of a given (diagonalized) vector $D=[d_1, \ldots, d_k]$.
\end{lemma}

If there is no risk of confusion, we write $UDV^\top$ as $A$ for the sake of readability. 

Using SVD, the Fourier expression is rewritten as follows:
\begin{align}
\text{\nxrefeq{fourier}}
&= \frac{1}{(2\pi)^k}\int_{M_{m,k} \times \RR^k} \gamma^\sharp(A,\oomega) \sigma^\sharp(\oomega)e^{i (UDV^\top\oomega) \cdot \xx} \delta(D) \dd U \dd D \dd V \dd \oomega 
\end{align}
{Changing the variables as $(\oomega,V) = (V\oomega',V)$ with $\dd \oomega \dd V= \dd \oomega'\dd V$,}
\begin{equation}
\phantom{\text{\nxrefeq{fourier}}}= \frac{1}{(2\pi)^k}\int_{M_{m,k} \times \RR^k} \gamma^\sharp(A,V\oomega') \sigma^\sharp(V\oomega')e^{i (UD\oomega') \cdot \xx} \delta(D) \dd U \dd D \dd V \dd \oomega'
\end{equation}
{Then, extending the domain of $D$ from $\RR_+^k$ to $\RR^k$, changing the variables as $(d_i,\omega'_i) = (y_i/\omega'_i,\omega'_i)$ with $\dd d_i \dd \omega'_i = \svert \omega'_i\svert ^{-1}\dd y_i \dd \omega'_i$, and writing $\yy \coloneq  [y_1, \ldots, y_k]$,}
\begin{equation}
\phantom{\text{\nxrefeq{fourier}}}= \frac{1}{(4\pi)^k}\int_{V_{m,k}\times\RR^k\times O(k) \times \RR^k} \gamma^\sharp(A,V\oomega') \sigma^\sharp(V\oomega')e^{i (U\yy) \cdot \xx} \delta(D) \prod_{i=1}^k \svert \omega'_i\svert ^{-1} \dd U \dd \yy \dd V \dd \oomega'. \xlabel{eq:svd2}
\end{equation}

\paragraph*{\textit{\step{3}}.}
Therefore, it is natural to suppose $\gamma$ to satisfy a \emph{separation-of-variables} form
\begin{equation}
    \gamma^\sharp(A,V\oomega') \sigma^\sharp(V\oomega') \delta(D) \prod_{i=1}^k \svert \omega'_i\svert ^{-1} = \widehat{f}(U\yy)\svert U\yy\svert ^{d}\phi^\sharp(V,\oomega') \xlabel{eq:sov}
\end{equation}
with an auxiliary convergence factor $\phi$. 
Then, we have
\begin{align*}
    \text{\nxrefeq{svd2}}
    &= \frac{1}{(4\pi)^k}\left( \int_{O(k) \times \RR^k} \phi^\sharp(V,\oomega')\dd V \dd \oomega' \right)\left(\int_{V_{m,k}\times\RR^k} \widehat{f}(U\yy)\svert U\yy\svert ^{d} e^{iU\yy \cdot \xx} \dd\yy \dd U \right) \\
    &= \frac{c_\phi}{(2\pi)^m} \int_{\RR^m} \widehat{f}(\yy) e^{i\yy\cdot\xx}\dd\yy
    = c_\phi f(\xx).
\end{align*}
Here, we put $c_\phi \coloneq  2^{-k} (2\pi)^{d} c_{m,k}^{-1} \int_{O(k) \times \RR^k} \phi^\sharp(V,\oomega')\dd V \dd \oomega'$, and used a matrix polar integration formula given in \nxreflem{mpd}.

Finally, the form \nxrefeq{sov} can be satisfied as below. Since $V\oomega'=\oomega$ and $U\yy = UD\oomega' = A\oomega$, it is reduced to 
\[
    \frac{\gamma^\sharp(A,\oomega)}{\phi^\sharp(V,\oomega')} = \frac{\widehat{f}(A\oomega)\svert A\oomega\svert ^{d}}{\sigma^\sharp(V\oomega') \delta(D) \prod_{i=1}^k \svert \omega'_i\svert ^{-1}}.
\]
Hence we can put
\begin{align}
    \gamma^\sharp(A,\oomega) &= \frac{\widehat{f}(A\oomega)\svert A\oomega\svert ^{d}\overline{\rho^\sharp(\oomega)}}{\delta(D)}, \xlabel{eq:gsharp}\\
    \phi^\sharp(V,\oomega') &= \sigma^\sharp(V\oomega') \overline{\rho^\sharp(V\oomega')} \prod_{i=1}^k \svert \omega'_i\svert ^{-1}, \notag
\end{align}
with additional convergence factor $\rho$.
In this setting, the constant $c_\phi$ is calculated as
\begin{align*}
    c_\phi
    &\coloneq  
    \frac{(2\pi)^{d}}{2^k c_{m,k}} \int_{O(k)\times\RR^k} \sigma^\sharp(V\oomega') \overline{\rho^\sharp(V\oomega')}\prod_{i=1}^k\svert  \omega'_i \svert ^{-1} \dd V \dd \oomega'  \\
    &= 
    \frac{(2\pi)^{d}}{2^k c_{m,k}} \int_{\RR^k} \sigma^\sharp(\oomega) \overline{\rho^\sharp(\oomega)}\prod_{i=1}^k\svert  \omega_i \svert ^{-1}\dd \oomega 
    \eqcolon  \iiprod{\sigma,\rho};
\end{align*}
and $\gamma$ is obtained by taking the Fourier inversion of \nxrefeq{gsharp} with respect to $\oomega$ as follows:
\begin{align*}
\gamma(A,\bb)
&= \frac{1}{(2\pi)^k \delta(D)} \int_{\RR^k} \svert A\oomega\svert ^{d} \widehat{f}(A\oomega) \overline{\rho^\sharp(\oomega)}e^{i\oomega \cdot \bb} \dd\oomega, \\
&= \frac{1}{(2\pi)^k \delta(D)} \int_{\RR^k} \left[\int_{\RR^m} \triangle^{d/2} [f](\xx) e^{-iA\oomega\cdot\xx} \dd \xx \right]\overline{\rho^\sharp(\oomega)}e^{i\oomega \cdot \bb} \dd\oomega \\
&= \frac{1}{\delta(D)} \int_{\RR^m} \triangle^{d/2} [f](\xx) \left[ \frac{1}{(2\pi)^k} \int_{\RR^k} \rho^\sharp(\oomega)e^{i\oomega \cdot (A^\top\xx-\bb)} \dd\oomega\right]^{*} \dd\xx \\
&= \frac{1}{\delta(D)} \int_{\RR^m} \triangle^{d/2} [f](\xx)\overline{\rho(A^\top\xx-\bb)}\dd \xx \\
&\eqcolon  
R[f](A,\bb).
\end{align*}

To sum up, we have shown that
\[
    S[R[f]](\xx)
    = \int_{M_{m,k}\times\RR^k} R[f](A,\bb)\sigma(A^\top\xx-\bb) \dd A \dd \bb
    = \iiprod{\sigma,\rho} f(\xx).\qed
\]

\subsection{Restriction to the Similitude Group}

Let us consider the restricted case of the similitude group $GV_{m,k}$. Since it is a measure-zero subspace of $M_{m,k}$, we can obtain the different solution.

\paragraph*{\textit{\step{1}}.}
The continuous network and its Fourier expression are given as
\begin{align}
    S[\gamma](\xx)
    &\coloneq  \int_{GV_{m,k}\times\RR^k} \gamma(A,\bb) \sigma(A^\top \xx - \bb )\dd \mu(A) \dd \bb \notag \\
    &= \frac{1}{(2\pi)^k} \int_{GV_{m,k} \times\RR^k} \gamma^\sharp(A,\oomega) \sigma^\sharp(\oomega) e^{i (aU\oomega)\cdot\xx} \alpha(a) \dd a \dd U \dd \oomega \xlabel{eq:fourier.sim}
\end{align}

\paragraph*{\textit{\step{2}}.}
(Skipping the matrix decomposition of $A$ and) turning $\oomega$ into polar coordinates $\oomega = r \vv$ with $(r,\vv) \in \RR_+ \times \SS^{k-1}$ and $\dd \oomega = r^{k-1} \dd r \dd \vv$, yielding
\begin{equation*}
    \text{\nxrefeq{fourier.sim}}
    = \frac{1}{(2\pi)^k} \int_{GV_{m,k} \times\RR_+ \times \SS^{k-1}} \gamma^\sharp(A,r\vv) \sigma^\sharp(r\vv) e^{i (arU\vv)\cdot\xx} \alpha(a) r^{k-1} \dd a \dd U \dd r \dd \vv
    \end{equation*}
{Changing the variable $(a,r) = (y/r,r)$ with $\dd a \dd r = r^{-1} \dd y \dd r$,}
\begin{equation*}
    \phantom{\text{\nxrefeq{fourier.sim}}}= \frac{1}{(2\pi)^k} \int_{GV_{m,k} \times \RR_+ \times \SS^{k-1}} \gamma^\sharp(A,r\vv) \sigma^\sharp(r\vv) e^{i (yU\vv)\cdot\xx} \alpha(y/r) r^{k-2} \dd y \dd U \dd r \dd \vv 
    \end{equation*}
{and returning $y \vv$ into the Euclidean coordinate $\yy$ with $y^{k-1} \dd y \dd \vv = \dd \yy$,}
\begin{equation}
    \phantom{\text{\nxrefeq{fourier.sim}}}= \frac{1}{(2\pi)^k} \int_{GV_{m,k} \times \RR^k} \gamma^\sharp(A,r\yy/\svert \yy\svert ) \sigma^\sharp(r\yy/\svert \yy\svert ) e^{i (U\yy)\cdot\xx} \alpha(\svert \yy\svert /r) r^{k-2} \svert \yy\svert ^{-(k-1)}\dd\yy \dd U \dd r. \xlabel{eq:fourier.simsim}
\end{equation}

\paragraph*{\textit{\step{3}}.}
Since $r\yy/\svert \yy\svert  = r\vv = \oomega$ and $\svert \yy\svert /r = y/r = a$, supposing the separation-of-variables form as
\begin{equation}
    \gamma^\sharp(A,\oomega)\sigma^\sharp(\oomega)\alpha(a)a^{-\ell}\svert \oomega\svert ^{k-2-\ell}\svert \yy\svert ^{-(k-1-\ell)} = \widehat{f}(U\yy)\svert U\yy\svert ^{d}\phi^\sharp(\oomega), \xlabel{eq:sov.sim}
\end{equation}
for any number $\ell \in \RR$,
we have
\[
    \text{\nxrefeq{fourier.simsim}}
    = \frac{1}{(2\pi)^k} \left(\int \int_{\RR_+} \phi^\sharp(r) \dd r\right) \left( \int_{V_{m,k}\times\RR^k} \widehat{f}(U\yy)\svert U\yy\svert ^{d}e^{-(U\yy)\cdot\xx}\dd U \dd \yy \right) = c_\phi f(\xx).
\]

Note that in addition to $U\yy = A\oomega$, we have $\svert \yy\svert  = \svert U\yy\svert  = a\svert \oomega\svert $ and thus $\svert U\yy\svert ^{d} = \svert A\oomega\svert ^n a^{d-n}\svert \oomega\svert ^{d-n}$ for any number $n \in \RR$. Hence the condition is reduced to
\begin{align*}
    \frac{\gamma^\sharp(A,\oomega)}{\phi^\sharp(\oomega)}
    &= \frac{\widehat{f}(A\oomega)\svert A\oomega\svert ^{n+(k-1-\ell)} a^{d-n+\ell}\svert \oomega\svert ^{m-2k-n+2+\ell}}{\sigma^\sharp(\oomega)\alpha(a)}\\
    &= \frac{\widehat{f}(A\oomega)\svert A\oomega\svert ^s a^{m-s-1}\svert \oomega\svert ^{d-s+1}}{\sigma^\sharp(\oomega)\alpha(a)},
\end{align*}
where we put $s \coloneq  n+(k-1-\ell)$, which can be an arbitrary real number.
As a result, to satisfy \nxrefeq{sov.sim}, we may put
\begin{align*}
    \gamma^\sharp(A,\oomega) &= \widehat{f}(A\oomega)\svert A\oomega\svert ^{s}\overline{\rho^\sharp(\oomega)},\\
    \quad \alpha(a) &= a^{m-s-1}, \\
    \phi^\sharp(\oomega) &= \sigma^\sharp(\oomega)\overline{\rho^\sharp(\oomega)}\svert \oomega\svert ^{-(d-s+1)};
\end{align*}
which lead to
\begin{align*}
    R_s[f] &= \int_{\RR^m} [\triangle^{s/2} f](\xx) \overline{\rho(A^\top \xx - \bb)}\dd \xx, \\
    \iiprod{\sigma,\rho}_s & \propto \int_{\RR^k} \sigma^\sharp(\oomega)\overline{\rho^\sharp(\oomega)}\svert \oomega\svert ^{-(d-s+1)}\dd\oomega, \\
    S[ R_s[f] ](\xx) &= \int_{GV_{m,k}\times\RR^k} R_s[f](A,\bb) \sigma(A^\top\xx-\bb) a^{m-s-1} \dd a \dd U \dd \bb = \iiprod{\sigma,\rho}f(\xx). \qed
\end{align*}
By matching the order of the fractional derivative $\triangle^s$, $s = d$ corresponds to the SVD solution. On the other hand, by matching the weight $\svert \oomega\svert ^{-(d-s+1)}$, $k=1$ and $s=0$ exactly reproduces the classical result.

\subsection{Restriction to the Stiefel manifold}

Let us consider a further restricted case of the Stiefel manifold.

\paragraph*{\textit{\step{1}}.}
\begin{align*}
    S[\gamma](\xx)
    &\coloneq  \int_{V_{m,k}\times\RR^k} \gamma(U,\bb)\sigma(U^\top\xx-\bb)\dd U \dd \bb \\
    &= \frac{1}{(2\pi)^k} \int_{V_{m,k}\times\RR^k} \gamma^\sharp(U,\oomega)\sigma^\sharp(\oomega)e^{i(U\oomega)\cdot\xx}\dd U \dd \oomega. 
\end{align*}
Namely, the weight matrix parameter $U \in M_{m,k}$ now simply lies in the Stiefel manifold $V_{m,k}$, and thus it does not contain any scaling factor. We show that this formulation still admit solutions, provided that $\sigma$ is also appropriately restricted.

\paragraph*{\textit{\step{3}}.}
Skipping the rescaling step (Step~2), let us consider the separation-of-variables form:
\begin{equation}
    \gamma^\sharp(U,\oomega)\sigma^\sharp(\oomega) = \widehat{f}(U\oomega)\svert U\oomega\svert ^{d}\phi^\sharp(\oomega); \xlabel{eq:sov.hom}
\end{equation}
satisfied by
\begin{align*}
    \gamma^\sharp(U,\oomega) &= \widehat{f}(U\oomega)\svert U\oomega\svert ^{s}\overline{\rho^\sharp(\oomega)}, \\
    \phi^\sharp(\oomega) &= \sigma^\sharp(\oomega)\overline{\rho^\sharp(\oomega)}\svert \oomega\svert ^{s-d},
\end{align*}
for any real number $s \in \RR$. Here, we used $\svert U\oomega\svert =\svert \oomega\svert $.

In order \nxrefeq{sov.hom} to turn to a solution, it is sufficient when
\begin{equation}
    \phi^\sharp(\oomega) = c_{m,k}^{-1}(2\pi)^{-d}, \xlabel{eq:hom.adm}
\end{equation}
because then
\begin{align*}
    \text{\nxrefeq{sov.hom}}
    &= \frac{1}{(2\pi)^k} \int_{V_{m,k}\times\RR^k} \widehat{f}(U\oomega)\svert U\oomega\svert ^{d}\phi^\sharp(\oomega) e^{i(U\oomega)\cdot\xx}\dd U \dd \oomega \\
    &= \frac{1}{(2\pi)^m} \int_{\RR^m} \widehat{f}(\yy)e^{i\yy\cdot\xx}\dd\yy = f(\xx).\qed
\end{align*}
Compared to the previous results, \nxrefeq{hom.adm} demands much more strict. Nonetheless, a few examples are such as
\[
    \sigma^\sharp(\oomega) = \svert \oomega\svert ^t, \quad \rho^\sharp(\oomega) = c_{m,k}^{-1}(2\pi)^{-d} \svert \oomega\svert ^{d-(s+t)};\]
or equivalently in the real domain,
\[
    \triangle_{\bb}^{-\frac{t}{2}}[\sigma](\bb) = \delta(\bb), \quad \triangle_{\bb}^{-\frac{d-(s+t)}{2}}[\rho](\bb) = c_{m,k}^{-1}(2\pi)^{-d} \delta(\bb).
\]
In particular when $k=1$, then $\sigma$ coincides with the Dirac delta ($t=0$), step function ($t=-1$), and ReLU function ($t=-2$).

Interestingly, the ridgelet transform is reduced to the $d$-plane transform ($d \coloneq  m-k$ is the codimension). Since $\gamma^\sharp(U,\oomega) = \widehat{f}(U\oomega)\svert U\oomega\svert ^{d-t}$, we have
\begin{align*}
    \gamma(U,\bb)
    &= \frac{1}{(2\pi)^k} \int_{\RR^k} \widehat{f}(U\oomega)\svert U\oomega\svert ^{d-t} e^{i\oomega\cdot\bb}\dd\oomega \\
    &= \frac{1}{(2\pi)^k} \int_{\RR^k} \widehat{\triangle^{(d-t)/2}[f]}(U\oomega) e^{i\oomega\cdot\bb}\dd\oomega,
\end{align*}
but this is the Fourier expression (a.k.a. Fourier slice theorem) for the $d$-plane transform, say $P_d$, of the derivative $\triangle^{(d-t)/2}[f]$. In other words, when the scaling parameter is removed, the reconstruction formula is reduced to the Radon transform:
\[
    S[ R[f] ](\xx) = \int_{V_{m,k}\times\RR^k} P_d[  \triangle^{(d-t)/2}[f] ](U,\bb)\sigma(U^\top\xx-\bb)\dd U \dd \bb.
\]

\bibliographystyle{abbrvnat}
\begin{small}
\bibliography{libraryS}

\begin{thebibliography}{66}
\providecommand{\natexlab}[1]{#1}
\providecommand{\url}[1]{\texttt{#1}}
\expandafter\ifx\csname urlstyle\endcsname\relax
  \providecommand{\doi}[1]{doi: #1}\else
  \providecommand{\doi}{doi: \begingroup \urlstyle{rm}\Url}\fi

\bibitem[Barron(1993)]{Barron1993}
A.~R. Barron.
\newblock \href{http://doi.org/10.1109/18.256500}{{Universal approximation bounds for superpositions of a sigmoidal function}}.
\newblock \emph{IEEE Transactions on Information Theory}, 39\penalty0 (3):\penalty0 930--945, 1993.

\bibitem[Bartolucci et~al.(2021)Bartolucci, {De Mari}, and Monti]{Bartolucci2021}
F.~Bartolucci, F.~{De Mari}, and M.~Monti.
\newblock \href{http://doi.org/10.1007/978-3-030-86664-8_1}{{Unitarization of the Horocyclic Radon Transform on Symmetric Spaces}}.
\newblock In \emph{Harmonic and Applied Analysis: From Radon Transforms to Machine Learning}, pages 1--54. Springer International Publishing, Cham, 2021.

\bibitem[Bengio et~al.(2006)Bengio, {Le Roux}, Vincent, Delalleau, and Marcotte]{LeRoux2006.convex}
Y.~Bengio, N.~{Le Roux}, P.~Vincent, O.~Delalleau, and P.~Marcotte.
\newblock \href{https://papers.nips.cc/paper/2800-convex-neural-networks.pdf}{{Convex neural networks}}.
\newblock In \emph{Advances in Neural Information Processing Systems 18}, pages 123--130, 2006.

\bibitem[Bronstein et~al.(2021)Bronstein, Bruna, Cohen, and Veli{\v{c}}kovi{\'{c}}]{Bronstein2021}
M.~M. Bronstein, J.~Bruna, T.~Cohen, and P.~Veli{\v{c}}kovi{\'{c}}.
\newblock \href{http://arxiv.org/abs/2104.13478}{{Geometric Deep Learning: Grids, Groups, Graphs, Geodesics, and Gauges}}.
\newblock \emph{arXiv preprint: 2104.13478}, 2021.

\bibitem[Bruna and Mallat(2013)]{Bruna2013}
J.~Bruna and S.~Mallat.
\newblock \href{http://doi.org/10.1109/TPAMI.2012.230}{{Invariant scattering convolution networks}}.
\newblock \emph{IEEE Transactions on Pattern Analysis and Machine Intelligence}, 35\penalty0 (8):\penalty0 1872--1886, 2013.

\bibitem[Cand{\`{e}}s(1998)]{Candes.PhD}
E.~J. Cand{\`{e}}s.
\newblock \emph{\href{https://searchworks.stanford.edu/view/9949708}{{Ridgelets: theory and applications}}}.
\newblock PhD thesis, Standford University, 1998.

\bibitem[Carroll and Dickinson(1989)]{Carroll.Dickinson}
S.~M. Carroll and B.~W. Dickinson.
\newblock \href{http://doi.org/10.1109/IJCNN.1989.118639}{{Construction of neural nets using the Radon transform}}.
\newblock In \emph{International Joint Conference on Neural Networks 1989}, volume~1, pages 607--611. IEEE, 1989.

\bibitem[Chizat and Bach(2018)]{Chizat2018}
L.~Chizat and F.~Bach.
\newblock \href{https://papers.nips.cc/paper/7567-on-the-global-convergence-of-gradient-descent-for-over-parameterized-models-using-optimal-transport/}{{On the Global Convergence of Gradient Descent for Over-parameterized Models using Optimal Transport}}.
\newblock In \emph{Advances in Neural Information Processing Systems 32}, pages 3036--3046, 2018.

\bibitem[Cohen and Welling(2016)]{Cohen2016a}
T.~Cohen and M.~Welling.
\newblock \href{https://proceedings.mlr.press/v48/cohenc16.html}{{Group Equivariant Convolutional Networks}}.
\newblock In \emph{Proceedings of The 33rd International Conference on Machine Learning}, volume~48, pages 2990--2999, 2016.

\bibitem[Cohen et~al.(2019)Cohen, Geiger, and Weiler]{Cohen2019}
T.~S. Cohen, M.~Geiger, and M.~Weiler.
\newblock \href{https://proceedings.neurips.cc/paper/2019/file/b9cfe8b6042cf759dc4c0cccb27a6737-Paper.pdf}{{A General Theory of Equivariant CNNs on Homogeneous Spaces}}.
\newblock In \emph{Advances in Neural Information Processing Systems}, volume~32, 2019.

\bibitem[DeVore and Lorentz(1993)]{DeVore1993}
R.~A. DeVore and G.~G. Lorentz.
\newblock \emph{\href{https://www.springer.com/gp/book/9783540506270}{{Constructive Approximation}}}.
\newblock Springer-Verlag Berlin Heidelberg, 1993.

\bibitem[D{\'{i}}az-Garc{\'{i}}a and Gonz{\'{a}}lez-Far{\'{i}}as(2005)]{Diaz-Garcia2005}
J.~A. D{\'{i}}az-Garc{\'{i}}a and G.~Gonz{\'{a}}lez-Far{\'{i}}as.
\newblock \href{http://doi.org/10.1016/j.jmva.2004.03.002}{{Singular random matrix decompositions: Jacobians}}.
\newblock \emph{Journal of Multivariate Analysis}, 93\penalty0 (2):\penalty0 296--312, 2005.

\bibitem[Donoho(2001)]{Donoho.ridgelet}
D.~L. Donoho.
\newblock \href{http://doi.org/10.1006/jath.2001.3568}{{Ridge functions and orthonormal ridgelets}}.
\newblock \emph{Journal of Approximation Theory}, 111\penalty0 (2):\penalty0 143--179, 2001.

\bibitem[Donoho(2002)]{Donoho2002}
D.~L. Donoho.
\newblock \href{http://arxiv.org/abs/math/0212395}{{Emerging applications of geometric multiscale analysis}}.
\newblock \emph{Proceedings of the ICM, Beijing 2002}, I:\penalty0 209--233, 2002.

\bibitem[Fukushima(1980)]{Fukushima1980}
K.~Fukushima.
\newblock \href{http://doi.org/10.1007/BF00344251}{{Neocognitron: A self-organizing neural network model for a mechanism of pattern recognition unaffected by shift in position}}.
\newblock \emph{Biological Cybernetics}, 36\penalty0 (4):\penalty0 193--202, 1980.

\bibitem[Funahashi(1989)]{Funahashi1989}
K.-I. Funahashi.
\newblock \href{http://doi.org/10.1016/0893-6080(89)90003-8}{{On the approximate realization of continuous mappings by neural networks}}.
\newblock \emph{Neural Networks}, 2\penalty0 (3):\penalty0 183--192, 1989.

\bibitem[Ganea et~al.(2018)Ganea, Becigneul, and Hofmann]{Ganea2018hnn}
O.~Ganea, G.~Becigneul, and T.~Hofmann.
\newblock \href{https://proceedings.neurips.cc/paper/2018/file/dbab2adc8f9d078009ee3fa810bea142-Paper.pdf}{{Hyperbolic Neural Networks}}.
\newblock In \emph{Advances in Neural Information Processing Systems 31}, 2018.

\bibitem[Grafakos(2008)]{Grafakos.classic}
L.~Grafakos.
\newblock \emph{\href{http://doi.org/10.1007/978-0-387-09432-8}{{Classical Fourier Analysis}}}.
\newblock Springer New York, second edition, 2008.

\bibitem[Gulcehre et~al.(2019)Gulcehre, Denil, Malinowski, Razavi, Pascanu, Hermann, Battaglia, Bapst, Raposo, Santoro, and de~Freitas]{Gulcehre2019}
C.~Gulcehre, M.~Denil, M.~Malinowski, A.~Razavi, R.~Pascanu, K.~M. Hermann, P.~Battaglia, V.~Bapst, D.~Raposo, A.~Santoro, and N.~de~Freitas.
\newblock \href{https://openreview.net/forum?id=rJxHsjRqFQ}{{Hyperbolic Attention Networks}}.
\newblock In \emph{International Conference on Learning Representations}, 2019.

\bibitem[Helgason(1984)]{Helgason.GGA}
S.~Helgason.
\newblock \emph{\href{http://doi.org/10.1090/surv/083}{{Groups and Geometric Analysis: Integral Geometry, Invariant Differential Operators, and Spherical Functions}}}.
\newblock American Mathematical Society, 1984.

\bibitem[Helgason(2008)]{Helgason.GASS}
S.~Helgason.
\newblock \emph{\href{http://doi.org/10.1090/surv/039}{{Geometric Analysis on Symmetric Spaces: Second Edition}}}.
\newblock American Mathematical Society, second edition, 2008.

\bibitem[Helgason(2010)]{Helgason.new}
S.~Helgason.
\newblock \emph{\href{http://doi.org/10.1007/978-1-4419-6055-9}{{Integral Geometry and Radon Transforms}}}.
\newblock Springer-Verlag New York, 2010.

\bibitem[Irie and Miyake(1988)]{Irie1988}
B.~Irie and S.~Miyake.
\newblock \href{http://doi.org/10.1109/ICNN.1988.23901}{{Capabilities of three-layered perceptrons}}.
\newblock In \emph{IEEE 1988 International Conference on Neural Networks}, pages 641--648. IEEE, 1988.

\bibitem[Ito(1991)]{Ito.Radon}
Y.~Ito.
\newblock \href{http://doi.org/10.1016/0893-6080(91)90075-G}{{Representation of functions by superpositions of a step or sigmoid function and their applications to neural network theory}}.
\newblock \emph{Neural Networks}, 4\penalty0 (3):\penalty0 385--394, 1991.

\bibitem[Kainen et~al.(2013)Kainen, K\r{u}rkov\'{a}, and Sanguineti]{kainen.survey}
P.~C. Kainen, V.~K\r{u}rkov\'{a}, and M.~Sanguineti.
\newblock \href{http://doi.org/10.1007/978-3-642-36657-4}{{Approximating multivariable functions by feedforward neural nets}}.
\newblock In \emph{Handbook on Neural Information Processing}, volume~49 of \emph{Intelligent Systems Reference Library}, pages 143--181. Springer Berlin Heidelberg, 2013.

\bibitem[Kapovich et~al.(2017)Kapovich, Leeb, and Porti]{Kapovich2017}
M.~Kapovich, B.~Leeb, and J.~Porti.
\newblock \href{http://doi.org/10.1007/s40879-017-0192-y}{{Anosov subgroups: dynamical and geometric characterizations}}.
\newblock \emph{European Journal of Mathematics}, 3\penalty0 (4):\penalty0 808--898, 2017.

\bibitem[Kondor and Trivedi(2018)]{Kondor2018}
R.~Kondor and S.~Trivedi.
\newblock \href{https://proceedings.mlr.press/v80/kondor18a.html}{{On the Generalization of Equivariance and Convolution in Neural Networks to the Action of Compact Groups}}.
\newblock In \emph{Proceedings of the 35th International Conference on Machine Learning}, pages 2747--2755, 2018.

\bibitem[Kostadinova et~al.(2014)Kostadinova, Pilipovi{\'{c}}, Saneva, and Vindas]{Kostadinova2014}
S.~Kostadinova, S.~Pilipovi{\'{c}}, K.~Saneva, and J.~Vindas.
\newblock \href{http://doi.org/10.1080/10652469.2013.853057}{{The ridgelet transform of distributions}}.
\newblock \emph{Integral Transforms and Special Functions}, 25\penalty0 (5):\penalty0 344--358, 2014.

\bibitem[Krioukov et~al.(2010)Krioukov, Papadopoulos, Kitsak, Vahdat, and Bogu{\~{n}}{\'{a}}]{Krioukov2010}
D.~Krioukov, F.~Papadopoulos, M.~Kitsak, A.~Vahdat, and M.~Bogu{\~{n}}{\'{a}}.
\newblock \href{http://doi.org/10.1103/PhysRevE.82.036106}{{Hyperbolic geometry of complex networks}}.
\newblock \emph{Phys. Rev. E}, 82\penalty0 (3):\penalty0 36106, 2010.

\bibitem[Krizhevsky et~al.(2012)Krizhevsky, Sutskever, and Hinton]{Krizhevsky2012}
A.~Krizhevsky, I.~Sutskever, and G.~E. Hinton.
\newblock \href{http://papers.nips.cc/paper/4824-imagenet-classification-with-deep-convolutional-neural-networks.pdf}{{ImageNet Classification with Deep Convolutional Neural Networks}}.
\newblock In \emph{Advances in Neural Information Processing Systems 25}, pages 1097--1105, 2012.

\bibitem[Kumagai and Sannai(2020)]{Kumagai2020}
W.~Kumagai and A.~Sannai.
\newblock \href{http://arxiv.org/abs/2012.13882}{{Universal Approximation Theorem for Equivariant Maps by Group CNNs}}.
\newblock \emph{arXiv preprint: 2012.13882}, 2020.

\bibitem[LeCun et~al.(1998)LeCun, Bottou, Bengio, and Haffner]{LeCun1998}
Y.~LeCun, L.~Bottou, Y.~Bengio, and P.~Haffner.
\newblock \href{http://doi.org/10.1109/5.726791}{{Gradient-Based Learning Applied to Document Recognition}}.
\newblock \emph{Proceedings of the IEEE}, 86:\penalty0 2278--2324, 1998.

\bibitem[Maron et~al.(2019)Maron, Fetaya, Segol, and Lipman]{Maron2019universality}
H.~Maron, E.~Fetaya, N.~Segol, and Y.~Lipman.
\newblock \href{https://proceedings.mlr.press/v97/maron19a.html}{{On the Universality of Invariant Networks}}.
\newblock In \emph{Proceedings of the 36th International Conference on Machine Learning}, volume~97, pages 4363--4371, 2019.

\bibitem[Mei et~al.(2018)Mei, Montanari, and Nguyen]{Mei2018}
S.~Mei, A.~Montanari, and P.-M. Nguyen.
\newblock \href{http://doi.org/10.1073/PNAS.1806579115}{{A mean field view of the landscape of two-layer neural networks}}.
\newblock \emph{Proceedings of the National Academy of Sciences}, 115\penalty0 (33):\penalty0 E7665--E7671, 2018.

\bibitem[Murata(1996)]{Murata1996}
N.~Murata.
\newblock \href{http://doi.org/10.1016/0893-6080(96)00000-7}{{An integral representation of functions using three-layered networks and their approximation bounds}}.
\newblock \emph{Neural Networks}, 9\penalty0 (6):\penalty0 947--956, 1996.

\bibitem[Nickel and Kiela(2017)]{Nickel2017}
M.~Nickel and D.~Kiela.
\newblock \href{https://proceedings.neurips.cc/paper/2017/file/59dfa2df42d9e3d41f5b02bfc32229dd-Paper.pdf}{{Poincar{\'{e}} Embeddings for Learning Hierarchical Representations}}.
\newblock In \emph{Advances in Neural Information Processing Systems}, volume~30, 2017.

\bibitem[Nickel and Kiela(2018)]{Nickel2018}
M.~Nickel and D.~Kiela.
\newblock \href{http://proceedings.mlr.press/v80/nickel18a.html}{{Learning Continuous Hierarchies in the {L}orentz Model of Hyperbolic Geometry}}.
\newblock In \emph{Proceedings of the 35th International Conference on Machine Learning}, volume~80, pages 3779--3788, 2018.

\bibitem[Nitanda and Suzuki(2017)]{Nitanda2017}
A.~Nitanda and T.~Suzuki.
\newblock \href{http://arxiv.org/abs/1712.05438}{{Stochastic Particle Gradient Descent for Infinite Ensembles}}.
\newblock \emph{arXiv preprint: 1712.05438}, 2017.

\bibitem[Nitanda et~al.(2022)Nitanda, Wu, and Suzuki]{Nitanda2022langevin}
A.~Nitanda, D.~Wu, and T.~Suzuki.
\newblock \href{https://proceedings.mlr.press/v151/nitanda22a.html}{{Convex Analysis of the Mean Field Langevin Dynamics}}.
\newblock In \emph{Proceedings of The 25th International Conference on Artificial Intelligence and Statistics}, pages 9741--9757, 2022.

\bibitem[Ongie et~al.(2020)Ongie, Willett, Soudry, and Srebro]{Ongie2020}
G.~Ongie, R.~Willett, D.~Soudry, and N.~Srebro.
\newblock \href{https://openreview.net/forum?id=H1lNPxHKDH}{{A Function Space View of Bounded Norm Infinite Width ReLU Nets: The Multivariate Case}}.
\newblock In \emph{International Conference on Learning Representations}, 2020.

\bibitem[Parhi and Nowak(2021)]{Parhi2021}
R.~Parhi and R.~D. Nowak.
\newblock \href{http://jmlr.org/papers/v22/20-583.html}{{Banach Space Representer Theorems for Neural Networks and Ridge Splines}}.
\newblock \emph{Journal of Machine Learning Research}, 22\penalty0 (43):\penalty0 1--40, 2021.

\bibitem[Parhi and Unser(2023{\natexlab{a}})]{parhi2023distributional}
R.~Parhi and M.~Unser.
\newblock \href{https://arxiv.org/abs/2310.01233}{Distributional Extension and Invertibility of the $k$-Plane Transform and Its Dual}.
\newblock \emph{arXiv preprint: 2310.01233}, 2023{\natexlab{a}}.

\bibitem[Parhi and Unser(2023{\natexlab{b}})]{parhi2023functionspace}
R.~Parhi and M.~Unser.
\newblock \href{https://arxiv.org/abs/2310.03696}{Function-Space Optimality of Neural Architectures With Multivariate Nonlinearities}.
\newblock \emph{arXiv preprint: 2310.03696}, 2023{\natexlab{b}}.

\bibitem[Ranzato et~al.(2007)Ranzato, Poultney, Chopra, and LeCun]{Ranzato2006}
M.~Ranzato, C.~Poultney, S.~Chopra, and Y.~LeCun.
\newblock \href{https://papers.nips.cc/paper/3112-efficient-learning-of-sparse-representations-with-an-energy-based-model.pdf}{{Efficient Learning of Sparse Representations with an Energy-Based Model}}.
\newblock In \emph{Advances In Neural Information Processing Systems 19}, pages 1137--1144, 2007.

\bibitem[Rotskoff and Vanden-Eijnden(2018)]{Rotskoff2018}
G.~Rotskoff and E.~Vanden-Eijnden.
\newblock \href{https://proceedings.neurips.cc/paper_files/paper/2018/hash/196f5641aa9dc87067da4ff90fd81e7b-Abstract.html}{{Parameters as interacting particles: long time convergence and asymptotic error scaling of neural networks}}.
\newblock In \emph{Advances in Neural Information Processing Systems 31}, pages 7146--7155, 2018.

\bibitem[Rubin(2004)]{Rubin.ridgelet}
B.~Rubin.
\newblock \href{http://doi.org/10.1016/j.acha.2004.03.003}{{Convolution–backprojection method for the $k$-plane transform, and Calder{\'{o}}n's identity for ridgelet transforms}}.
\newblock \emph{Applied and Computational Harmonic Analysis}, 16\penalty0 (3):\penalty0 231--242, 2004.

\bibitem[Rubin(2018)]{Rubin2018}
B.~Rubin.
\newblock \href{http://doi.org/doi:10.1515/fca-2018-0086}{{A note on the Blaschke-Petkantschin formula, Riesz distributions, and Drury's identity}}.
\newblock \emph{Fractional Calculus and Applied Analysis}, 21\penalty0 (6):\penalty0 1641--1650, 2018.

\bibitem[Sala et~al.(2018)Sala, {De Sa}, Gu, and Re]{Sala2018}
F.~Sala, C.~{De Sa}, A.~Gu, and C.~Re.
\newblock \href{http://proceedings.mlr.press/v80/sala18a.html}{{Representation Tradeoffs for Hyperbolic Embeddings}}.
\newblock In \emph{Proceedings of the 35th International Conference on Machine Learning}, volume~80, pages 4460--4469, 2018.

\bibitem[Savarese et~al.(2019)Savarese, Evron, Soudry, and Srebro]{Savarese2019}
P.~Savarese, I.~Evron, D.~Soudry, and N.~Srebro.
\newblock \href{http://proceedings.mlr.press/v99/savarese19a.html}{{How do infinite width bounded norm networks look in function space?}}
\newblock In \emph{Proceedings of the 32nd Conference on Learning Theory}, volume~99, pages 2667--2690, 2019.

\bibitem[Shimizu et~al.(2021)Shimizu, Mukuta, and Harada]{Shimizu2021}
R.~Shimizu, Y.~Mukuta, and T.~Harada.
\newblock \href{https://openreview.net/forum?id=Ec85b0tUwbA}{{Hyperbolic Neural Networks++}}.
\newblock In \emph{International Conference on Learning Representations}, 2021.

\bibitem[Sirignano and Spiliopoulos(2020)]{Sirignano2020lln}
J.~Sirignano and K.~Spiliopoulos.
\newblock \href{http://doi.org/10.1137/18M1192184}{{Mean Field Analysis of Neural Networks: A Law of Large Numbers}}.
\newblock \emph{SIAM Journal on Applied Mathematics}, 80\penalty0 (2):\penalty0 725--752, 2020.

\bibitem[Sonoda and Murata(2017)]{Sonoda2015acha}
S.~Sonoda and N.~Murata.
\newblock \href{http://doi.org/10.1016/j.acha.2015.12.005}{{Neural network with unbounded activation functions is universal approximator}}.
\newblock \emph{Applied and Computational Harmonic Analysis}, 43\penalty0 (2):\penalty0 233--268, 2017.

\bibitem[Sonoda et~al.(2021{\natexlab{a}})Sonoda, Ishikawa, and Ikeda]{Sonoda2021aistats}
S.~Sonoda, I.~Ishikawa, and M.~Ikeda.
\newblock \href{http://proceedings.mlr.press/v130/sonoda21a.html}{{Ridge Regression with Over-Parametrized Two-Layer Networks Converge to Ridgelet Spectrum}}.
\newblock In \emph{Proceedings of The 24th International Conference on Artificial Intelligence and Statistics 2021}, volume 130, pages 2674--2682, 2021{\natexlab{a}}.

\bibitem[Sonoda et~al.(2021{\natexlab{b}})Sonoda, Ishikawa, and Ikeda]{Sonoda2021ghost}
S.~Sonoda, I.~Ishikawa, and M.~Ikeda.
\newblock \href{http://arxiv.org/abs/2106.04770}{{Ghosts in Neural Networks: Existence, Structure and Role of Infinite-Dimensional Null Space}}.
\newblock \emph{arXiv preprint: 2106.04770}, 2021{\natexlab{b}}.

\bibitem[Sonoda et~al.(2022{\natexlab{a}})Sonoda, Ishikawa, and Ikeda]{Sonoda2022gconv}
S.~Sonoda, I.~Ishikawa, and M.~Ikeda.
\newblock \href{https://papers.nips.cc/paper_files/paper/2022/hash/fcc3dc27672a12510babe448d665e152-Abstract-Conference.html}{{Universality of Group Convolutional Neural Networks Based on Ridgelet Analysis on Groups}}.
\newblock In \emph{Advances in Neural Information Processing Systems 35}, pages 38680--38694, 2022{\natexlab{a}}.

\bibitem[Sonoda et~al.(2022{\natexlab{b}})Sonoda, Ishikawa, and Ikeda]{sonoda2022symmetric}
S.~Sonoda, I.~Ishikawa, and M.~Ikeda.
\newblock \href{https://proceedings.mlr.press/v162/sonoda22a.html}{{Fully-Connected Network on Noncompact Symmetric Space and Ridgelet Transform based on Helgason-Fourier Analysis}}.
\newblock In \emph{Proceedings of the 39th International Conference on Machine Learning}, volume 162, pages 20405--20422, 2022{\natexlab{b}}.

\bibitem[Sonoda et~al.(2023{\natexlab{a}})Sonoda, Hashimoto, Ishikawa, and Ikeda]{sonoda2023deepridge}
S.~Sonoda, Y.~Hashimoto, I.~Ishikawa, and M.~Ikeda.
\newblock \href{https://arxiv.org/abs/2310.03529}{{Deep Ridgelet Transform: Voice with Koopman Operator Proves Universality of Formal Deep Networks}}.
\newblock In \emph{Proceedings of the 2nd NeurIPS Workshop on Symmetry and Geometry in Neural Representations}, 2023{\natexlab{a}}.

\bibitem[Sonoda et~al.(2023{\natexlab{b}})Sonoda, Ishi, Ishikawa, and Ikeda]{sonoda2023joint}
S.~Sonoda, H.~Ishi, I.~Ishikawa, and M.~Ikeda.
\newblock \href{https://arxiv.org/abs/2310.03530}{{Joint Group Invariant Functions on Data-Parameter Domain Induce Universal Neural Networks}}.
\newblock In \emph{Proceedings of the 2nd NeurIPS Workshop on Symmetry and Geometry in Neural Representations}, 2023{\natexlab{b}}.

\bibitem[Starck et~al.(2010)Starck, Murtagh, and Fadili]{Starck2010}
J.-L. Starck, F.~Murtagh, and J.~M. Fadili.
\newblock \href{http://doi.org/10.1017/CBO9780511730344.006}{{The ridgelet and curvelet transforms}}.
\newblock In \emph{Sparse Image and Signal Processing: Wavelets, Curvelets, Morphological Diversity}, pages 89--118. Cambridge University Press, 2010.

\bibitem[Suzuki(2020)]{Suzuki2020langevin}
T.~Suzuki.
\newblock \href{https://proceedings.neurips.cc/paper/2020/hash/df1a336b7e0b0cb186de6e66800c43a9-Abstract.html}{{Generalization bound of globally optimal non-convex neural network training: Transportation map estimation by infinite dimensional Langevin dynamics}}.
\newblock In \emph{Advances in Neural Information Processing Systems 33}, pages 19224--19237, 2020.

\bibitem[Terras(2016)]{Terras2016}
A.~Terras.
\newblock \emph{\href{http://doi.org/10.1007/978-1-4939-3408-9}{{Harmonic Analysis on Symmetric Spaces—Higher Rank Spaces, Positive Definite Matrix Space and Generalizations}}}.
\newblock Springer New York, 2016.

\bibitem[Unser(2019)]{Unser2019}
M.~Unser.
\newblock \href{http://jmlr.org/papers/v20/18-418.html}{{A Representer Theorem for Deep Neural Networks}}.
\newblock \emph{Journal of Machine Learning Research}, 20\penalty0 (110):\penalty0 1--30, 2019.

\bibitem[Yamasaki et~al.(2023)Yamasaki, Subramanian, Hayakawa, and Sonoda]{Yamasaki2023icml}
H.~Yamasaki, S.~Subramanian, S.~Hayakawa, and S.~Sonoda.
\newblock \href{https://proceedings.mlr.press/v202/yamasaki23a.html}{{Quantum Ridgelet Transform: Winning Lottery Ticket of Neural Networks with Quantum Computation}}.
\newblock In \emph{Proceedings of the 40th International Conference on Machine Learning}, volume 202, pages 39008--39034, 2023.

\bibitem[Yarotsky(2022)]{Yarotsky2021a}
D.~Yarotsky.
\newblock \href{http://doi.org/10.1007/s00365-021-09546-1}{{Universal Approximations of Invariant Maps by Neural Networks}}.
\newblock \emph{Constructive Approximation}, 55:\penalty0 407--474, 2022.

\bibitem[Zaheer et~al.(2017)Zaheer, Kottur, Ravanbakhsh, Poczos, Salakhutdinov, and Smola]{Zaheer2017}
M.~Zaheer, S.~Kottur, S.~Ravanbakhsh, B.~Poczos, R.~R. Salakhutdinov, and A.~J. Smola.
\newblock \href{https://papers.nips.cc/paper_files/paper/2017/hash/f22e4747da1aa27e363d86d40ff442fe-Abstract.html}{Deep Sets}.
\newblock In \emph{Advances in Neural Information Processing Systems}, volume~30, 2017.

\bibitem[Zhou(2020)]{Zhou2020a}
D.-X. Zhou.
\newblock \href{http://doi.org/https://doi.org/10.1016/j.acha.2019.06.004}{{Universality of deep convolutional neural networks}}.
\newblock \emph{Applied and Computational Harmonic Analysis}, 48\penalty0 (2):\penalty0 787--794, 2020.

\end{thebibliography}
\end{small}

\end{document}